\documentclass{article}
\usepackage[T1]{fontenc}
\usepackage{iclr2027_conference,times}

\usepackage{amsmath,amssymb,mathtools,amsthm}
\usepackage{booktabs,multirow,tabularx}
\usepackage{siunitx}
\usepackage{graphicx}
\usepackage{float}
\usepackage{placeins}
\usepackage{xcolor}
\usepackage{colortbl}
\usepackage{tikz}
\usetikzlibrary{arrows.meta,positioning,fit,backgrounds}
\usepackage{hyperref}
\definecolor{referenceblue}{HTML}{2B608A}
\hypersetup{colorlinks=true,
  citecolor=referenceblue,linkcolor=referenceblue,urlcolor=referenceblue,
  pdftitle={HotPOT: Operator Transformer with Heterogeneous Physics Experts for Large-Scale PDE Pre-training},
  pdfauthor={Jiayi Li, Guancheng Zhou, Yisi Luo, Ruihua Chen, Penghao Jiang, Deyu Meng, Flora D. Salim, Piotr Koniusz},pdfsubject={},pdfkeywords={}}
\usepackage{xurl}
\usepackage[nameinlink,capitalise]{cleveref}

\newcommand{\R}{\mathbb{R}}
\newcommand{\Hist}{\mathbf{H}}
\newcommand{\Data}{\mathcal{D}}
\newcommand{\FeatureMap}[2]{f_{\mathrm{feat},#1,#2}}
\newcommand{\NumericalMap}[1]{f_{\mathrm{upd},#1}}
\newcommand{\DataConfigs}{\mathcal{R}}
\newcommand{\SystemTypes}{\mathcal{T}}
\newcommand{\UpdateSpace}{\mathcal{U}}
\newcommand{\Loss}{\ell_{\mathrm{train}}(\mathbf{w})}
\newcommand{\Shared}{f_{\mathrm{enc},\boldsymbol{\theta}}}
\newcommand{\Mixer}[1]{f_{\mathrm{mix},\boldsymbol{\phi}_{#1}}}
\newcommand{\Decoder}[1]{f_{\mathrm{dec},\boldsymbol{\psi}_{#1}}}
\newcommand{\state}{\mathbf{q}}
\newcommand{\update}{\mathbf{u}}
\newcommand{\latent}{\mathbf{Z}}
\newcommand{\context}{\mathbf{Z}^{\mathrm{mix}}}
\newcommand{\systemtype}{\tau}

\DeclareMathOperator{\concat}{concat}
\DeclareMathOperator{\stackrows}{stack}
\DeclareMathOperator{\diag}{diag}

\definecolor{sharedblue}{HTML}{4477AA}
\definecolor{privateorange}{HTML}{EE7733}
\definecolor{programgreen}{HTML}{228833}
\definecolor{chartgray}{HTML}{667788}
\definecolor{bestblue}{HTML}{1769AA}

\newtheorem{proposition}{Proposition}

\title{HotPOT: Operator Transformer with\\Heterogeneous Physics Experts for\\Large-Scale PDE Pre-training}

\author{\begin{minipage}{\dimexpr\textwidth-2\tabcolsep\relax}\centering\normalfont
Jiayi Li\textsuperscript{1,*}, Guancheng Zhou\textsuperscript{2,*}, Yisi Luo\textsuperscript{2}, Ruihua Chen\textsuperscript{2},\\
Penghao Jiang\textsuperscript{1}, Deyu Meng\textsuperscript{2,$\dagger$}, Flora D. Salim\textsuperscript{1,$\dagger$}, Piotr Koniusz\textsuperscript{1,3,$\dagger$}\\[0.6em]
{\small \textsuperscript{1}University of New South Wales\\
\textsuperscript{2}Xi'an Jiaotong University\\
\textsuperscript{3}Commonwealth Scientific and Industrial Research Organisation}
\end{minipage}}

\newcommand{\edit}[1]{#1}

\iclrfinalcopy
\begin{document}
\maketitle
\lhead{}
\renewcommand{\headrulewidth}{0pt}
\begingroup
\renewcommand{\thefootnote}{\fnsymbol{footnote}}
\interfootnotelinepenalty=10000
\footnotetext[1]{Equal contribution.}
\footnotetext[2]{\raggedright Corresponding authors: Deyu Meng (\href{mailto:dymeng@mail.xjtu.edu.cn}{dymeng@mail.xjtu.edu.cn}), Flora D. Salim (\href{mailto:flora.salim@unsw.edu.au}{flora.salim@unsw.edu.au}), and Piotr Koniusz (\href{mailto:piotr.koniusz@unsw.edu.au}{piotr.koniusz@unsw.edu.au}).}
\endgroup
\raggedbottom

\begin{abstract}
Large-scale PDE pre-training trains a single model on the data from multiple partial differential equations (PDEs). However, adding heterogeneous PDEs with different governing equations, physical variables, and variable interactions does not always improve predictive performance across tasks. Existing PDE mixture-of-experts models mitigate cross-task interference through learned routing, but typically use homogeneous expert architectures, making it difficult to directly reflect differences in state evolution across equations.
We introduce HotPOT, an operator transformer with heterogeneous physics experts for large-scale PDE pre-training. HotPOT uses a shared encoder to capture information common across PDEs, while each PDE family uses a neural--numerical expert that pairs a learnable predictor with a fixed numerical update containing no trainable parameters. Each predictor outputs the quantities required by its paired numerical update, which then computes the next state. This design allows different PDEs to share a common encoder while retaining their own learnable predictors and numerical updates. 
Experiments across five PDE families and six settings show that HotPOT outperforms MoE-POT in all 24 matched-scale comparisons before and after fine-tuning, reducing rollout error by 50.7\% on average. Notably, the 17M model also surpasses substantially larger MoE-POT models, and the performance gains extend to zero-shot prediction under unseen physical conditions.
Code is available at \url{https://github.com/cruiseresearchgroup/HotPOT-code}.

\end{abstract}

\section{Introduction}
\label{sec:introduction}

\edit{Partial differential equations (PDEs) are fundamental to modeling a wide range of physical systems. The high cost of repeated numerical simulations has motivated learning-based PDE solvers, particularly neural operators \citep{li2021fno,hmida2026compno} that learn solution mappings directly from data. 
Recent methods have shown that a single model can be trained jointly on data from multiple PDEs, rather than trained separately for each task, and have achieved strong performance across multiple PDE tasks.}

\edit{However, adding data from heterogeneous PDEs with different governing equations, physical variables, and variable interactions does not always improve performance on existing tasks. For example, the MoE-POT model proposed in \citet{wang2025moepot} reports increased error on some original tasks when expanding the training mixture of a dense pre-trained model from six to twelve datasets. To mitigate cross-task interference, MoE-POT uses learned routing to select specialized experts. These experts nevertheless use homogeneous architectures, which may be difficult to directly learn inherent physical distinctions in state evolution across equations.}

On the other hand, Neural--numerical methods offer a basis for constructing heterogeneous experts by combining learned predictions with explicit numerical operations \citep{karlbauer2022finn,gopakumar2026opssplit}. Different PDEs can require different update operations and prediction quantities. An expert may predict a rate for time integration, an increment for a constrained update, or state components for direct assignment. Pairing a predictor with the corresponding update allows the expert's computation to reflect the requirements of its PDE. This motivates a neural--numerical pre-training paradigm in which each PDE pairs learnable predictions with explicit numerical operations according to its update requirements, while shared history representations integrate the resulting experts into a unified pre-training model.

Specifically, we introduce HotPOT, an operator transformer that combines a shared encoder with heterogeneous neural--numerical experts for large-scale PDE pre-training. \Cref{fig:overview}(a) summarizes the workflow from heterogeneous PDE pre-training to prediction and transfer. The shared encoder processes scalar-field histories across PDEs, while each expert handles system-specific variable coupling and state evolution. Each expert pairs a learnable predictor with a PDE-specific fixed numerical update, which contains no trainable parameters. The predictor outputs the quantities required by its paired numerical update, which then computes the next state. The shared encoder and expert predictors are jointly trained using next-state supervision, bringing the different neural--numerical experts together in a single pre-training model.
The Navier--Stokes (NS) example in \cref{fig:overview}(b) illustrates the resulting prediction: at the tenth forecast step, HotPOT follows the reference vorticity contours more closely than MoE-POT in the displayed central region.

The main contributions are:
\begin{itemize}
  \item We propose HotPOT, an operator transformer that leverages PDE-specific experts containing learnable predictors with fixed numerical updates, allowing different equations to use different prediction quantities and update operations.
  \item We integrate these heterogeneous experts through a shared history encoder into one multi-PDE model. The encoder and predictors are trained end-to-end using only next-state supervision, while variable coupling and state evolution remain system-specific.
  \item We evaluate HotPOT across five PDE families and six settings. Relative to MoE-POT, HotPOT reduces rollout error by 50.7\% on average across 24 matched-scale comparisons before and after fine-tuning, and further improves zero-shot prediction under unseen physical conditions.
\end{itemize}

{\setlength{\intextsep}{6pt}\setlength{\abovecaptionskip}{6pt}
\begin{figure}[H]
\centering
\includegraphics[width=\textwidth]{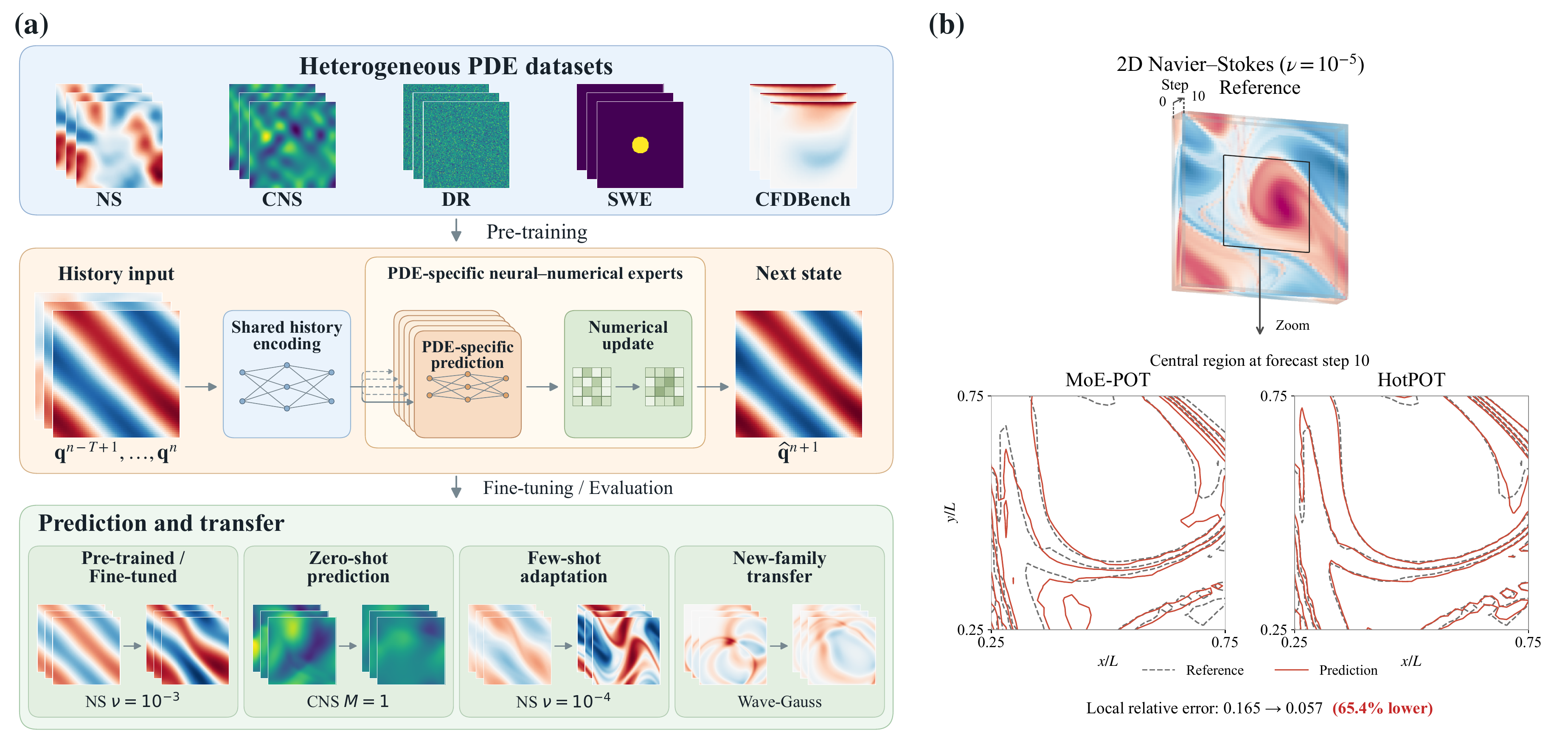}
\caption{\textbf{HotPOT: framework overview and an NS prediction example.}
(a) Shared history encoding connects PDE-specific predictor--update pairs for
pre-training and transfer. The state is $\state^n$ at step $n$, with history length $T=10$.
(b) Pre-trained 17M HotPOT and MoE-POT on the same NS test
trajectory ($\nu=10^{-5}$). The boxed region is compared at forecast step~10;
gray dashed and red solid curves denote reference and
predicted contours, respectively. The relative $L_2$ error in this region
decreases from 0.165 for MoE-POT to 0.057 for HotPOT (65.4\% lower).
Here $\nu$ is viscosity, $M$ is Mach number, and $L$ is domain length.}
\label{fig:overview}
\end{figure}
}

\section{Related Work}
\label{sec:related_work}

\paragraph{Multi-PDE pre-training.}
Previous studies including MPP \citep{mccabe2024mpp}, DPOT \citep{hao2024dpot}, Poseidon \citep{herde2024poseidon}, and Walrus \citep{mccabe2026walrus} learn reusable representations through
joint training across PDE systems.
Some related approaches propose to better organize representations around physical variables or
scalar fields, as in CoDA-NO \citep{rahman2024codano} and Tadpole \citep{liu2026tadpole}, or condition prediction on equation
and system information, as in PDEformer \citep{ye2024pdeformer}, PROSE-PDE \citep{sun2024prosepde}, and Unisolver \citep{zhou2025unisolver}.
These studies provide the foundation for learning across heterogeneous
PDE data. We focus on incorporating PDE-specific prediction and update
operations into this shared pre-training framework.

\paragraph{Mixture-of-experts for PDEs.}
MoE-POT \citep{wang2025moepot} combines shared experts with input-dependent routed experts,
while NESTOR \citep{sun2026nestor} routes computation at both image and token levels. These methods allocate specialized
neural capacity to different inputs through learned routing, while
retaining a common expert structure. HotPOT extends expert
specialization to the quantities predicted and the operations used to
evolve the state. Each PDE family uses a specific neural predictor paired with
its corresponding fixed numerical update. A shared history encoder
connects these heterogeneous experts within one jointly trained model.

\paragraph{Neural--numerical methods.}
Neural--numerical methods place learned components within structured
updates. Universal differential equations \citep{rackauckas2021ude} and APHYNITY \citep{yin2021aphynity} augment known
dynamics, while learned discretizations and solver-in-the-loop methods
learn numerical operators or corrections
\citep{barsinai2019datadriven,
kochkov2021mlcfd,um2020solver}.
FINN \citep{karlbauer2022finn} embeds trainable components in PDE-specific finite-volume structures,
and OpsSplit \citep{gopakumar2026opssplit} combines learned nonlinear operators with fixed
finite-difference operators and time integration.
CompNO \citep{hmida2026compno} combines pretrained operators with boundary enforcement,
and ComPhy \citep{trenta2026comphy} couples equation-specific modules.
Origo \citep{sun2026origo} decomposes evolution into global spectral and local constitutive
operators, while \citet{serrano2026splitting} compose learned operators at test time.
DISCO \citep{morel2025disco} generates a trajectory-conditioned operator for time integration.
Building on this line of work, HotPOT organizes PDE-specific
predictor--update pairs as heterogeneous experts for multi-PDE pre-training.
Their prediction quantities and update operations differ across PDEs,
while the shared encoder and expert predictors are jointly trained using
next-state supervision.

\section{Problem Setting}
\label{sec:problem}

\edit{We consider joint pre-training over a collection of time-dependent PDE systems. Let \(\mathcal R\) denote the set of prediction settings, and let \(r\in\mathcal R\) index one setting. A setting specifies a physical system together with its benchmark configuration, such as the observed variables, equation parameters, geometry, or boundary conditions. Different settings may therefore have different state variables, variable couplings, and state-transition mechanisms.}

\edit{Each setting $r$ has a trajectory dataset $\Data_r$ with $C_r$
observed physical fields. Flattening the spatial grid in a fixed order gives
$\state_{r,c}^n\in\R^{N_r}$ for field $c$ at saved step $n$, where $N_r$
is the number of grid cells and $q_{r,c,i}^n$ is its scalar value at cell $i$.
The operator $\concat$ concatenates its entries in the specified order. The full state is the vector
$\state_r^n=\concat(\{\state_{r,c}^n\}_{c=1}^{C_r})\in\R^{N_r C_r}$.
Geometry and support channels describe the spatial domain and valid grid
locations. They are provided as auxiliary inputs rather than physical fields
to be predicted, so they are excluded from $C_r$ and the state vector
$\state_r^n$. The pre-training corpus is
$\Data=\bigcup_{r\in\DataConfigs}\Data_r$.}

\edit{Let $g_{\mathbf w}$ denote the next-state predictor with trainable parameter
vector $\mathbf w$. The operator $\stackrows$ stacks its indexed row vectors
in order. Let $\widehat{\state}_r^{n+1}$ denote the predicted next state.
The most recent $T=10$ saved states form the history matrix $\Hist_r^n$:
\begin{equation}
 \Hist_r^n=\stackrows\!\left(\{(\state_r^{n-T+j})^\top\}_{j=1}^{T}\right),
 \qquad
 \widehat{\state}_r^{n+1}=g_{\mathbf w}(\Hist_r^n;r).
 \label{eq:target_transition}
\end{equation}
Thus $\Hist_r^n\in\R^{T\times N_r C_r}$ contains states in chronological
order. The matrix $\Hist_{r,c}^n\in\R^{T\times N_r}$ is its field-$c$ block.
The argument $r$ identifies the PDE setting. The predictor also uses the
geometry and support information, which we omit from the arguments of
$g_{\mathbf w}$ to keep the notation compact. Each prediction is appended
to the history for autoregressive rollout. We use bold lowercase symbols
for vectors, bold uppercase symbols for matrices, and calligraphic uppercase
symbols for sets or spaces. Indexed braces gather entries in index order.
Parentheses group arguments or pairs.}

\section{Method}
\label{sec:method}

HotPOT builds its shared encoder and neural predictors using patch encoding,
temporal aggregation, and grouped Fourier/local blocks from DPOT and
MoE-POT \citep{hao2024dpot,wang2025moepot}.
As shown in \cref{fig:architecture}, HotPOT realizes $g_{\mathbf w}$ by
connecting PDE-family predictor--update pairs through a shared history encoder.

Let $\systemtype:\DataConfigs\to\SystemTypes$ map settings to the set of expert families. The family $\systemtype(r)$ is
assigned to setting $r$. Settings from the same PDE family use the
corresponding predictor and numerical update program. In particular, the
two NS viscosity settings share one NS expert family.
The shared encoder has parameter vector $\boldsymbol\theta$. For family
$e\in\SystemTypes$, $\boldsymbol\phi_e$ parameterizes field coupling and
evolution, and $\boldsymbol\psi_e$ parameterizes decoding and readout.
Their aggregation is $\mathbf w=\concat(\boldsymbol\theta,
\{(\boldsymbol\phi_e,\boldsymbol\psi_e)\}_{e\in\SystemTypes})$,
where parameter pairs are flattened and concatenated in a fixed family order.

{\setlength{\intextsep}{6pt}\setlength{\abovecaptionskip}{6pt}
\begin{figure}[!htb]
\centering
\input{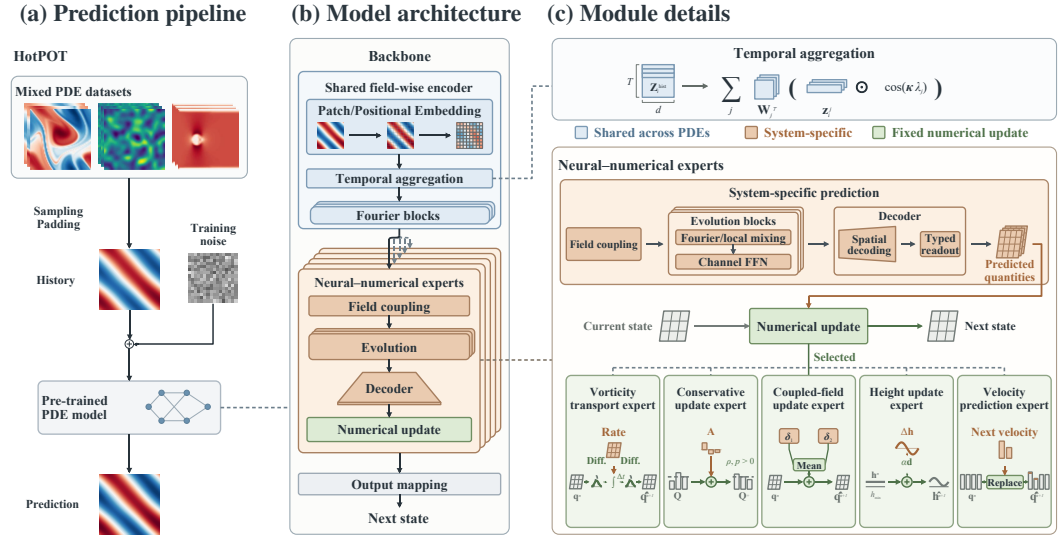}
\caption{\textbf{Architecture of HotPOT.}
\textbf{(a)} Prediction pipeline with sampling, padding, and training-time noise.
\textbf{(b)} A shared encoder connects PDE-specific neural--numerical experts
in one model.
\textbf{(c)} Temporal aggregation and five alternative numerical updates.
Blue, orange, green, and gray indicate shared computation, system-specific
prediction, fixed numerical updates, and physical states, respectively.
Stage-state queries, history statistics, and repeated readouts are omitted.
FFN denotes feed-forward network. Field images are illustrative.
Notation and module details are given in
Appendices~\ref{app:backbone} and~\ref{app:expert_view}.}
\label{fig:architecture}
\end{figure}
}

\subsection{Shared Field-Wise History Representation}

The deterministic function $\FeatureMap{r}{c}$ normalizes field $c$'s
history and constructs six input features: normalized field value, two
coordinates, boundary distance, geometry support, and field validity.
This mapping does not mix physical fields. Each scalar history then passes
through the same encoder $\Shared$, with no learned cross-field mixing.
At saved time $n$, let $\latent_{r,c}\in\R^{P\times d_{\mathrm{enc}}}$
denote its representation, where $P$ is the number of spatial patches and
$d_{\mathrm{enc}}$ the shared width. Then
\begin{equation}
  \latent_{r,c}
  =\Shared\!\left(\FeatureMap{r}{c}(\Hist_{r,c}^n)\right),
  \qquad c=1,\ldots,C_r.
  \label{eq:shared_analysis}
\end{equation}
A shared patch encoder and positional embedding
produce tokens for each history position. Temporal aggregation then combines
these tokens within each field. Shared Fourier/local and channel feed-forward
network (FFN) blocks produce this feature matrix. The saved-time index $n$
is suppressed.
Next-state supervision jointly trains these features across datasets.
Appendix~\ref{app:backbone} gives the aggregation formula, dimensions, and
block equations.

\subsection{System-Specific Prediction}
\label{sec:system_prediction}

Each PDE family's predictor supplies the quantities required by its numerical
update: transport rates for NS, conservative-variable actions for CNS, height
increments for SWE, concentration-increment proposals for DR, or next-state
velocities for CFDB. We denote this prediction space by $\UpdateSpace_r$.

\paragraph{Field coupling and evolution.}
Let $\Mixer{\systemtype(r)}$ denote field coupling followed by system-specific
evolution. For field $c$, let $\mu_{r,c}$ and $\sigma_{r,c}$ be the history mean
and floored standard deviation over space and the $T$ input frames, and collect
them as $\mathbf s_{r,c}=(\mu_{r,c},\log\sigma_{r,c})^\top\in\R^2$.
The coupled context $\context_r\in\R^{P\times d_{\mathrm{priv}}}$, with private
width $d_{\mathrm{priv}}$, is
\begin{equation}
  \context_r=\Mixer{\systemtype(r)}
  \!\left(\{(\latent_{r,c},\mathbf s_{r,c})\}_{c=1}^{C_r}\right).
  \label{eq:private_mixer}
\end{equation}
Coupling concatenates field representations, applies a $1\times1$ projection,
and adds history-statistic and positional embeddings. System-specific
Fourier/local/FFN blocks then produce the context.

\paragraph{Spatial decoding and prediction.}
The decoder $\Decoder{\systemtype(r)}$ expands $\context_r$ to full resolution
using a transposed convolution, Gaussian error linear unit (GELU) activation,
and pointwise projection. Its readout combines this cached context with a
query state $\boldsymbol\xi_r^{n,k}\in\R^{d_{\xi,r}}$, representing the current
or an intermediate state. Here $k=1,\ldots,K_r$ indexes queries within the
saved interval $\Delta t_r$, the time between consecutive saved states.
Let $\widehat{\update}_r^{n,k}\in\UpdateSpace_r\subseteq\R^{d_{u,r}}$ denote
the neural prediction at query $k$; $d_{\xi,r}$ and $d_{u,r}$ count the spatial
and channel entries of the flattened query and prediction, respectively.
Then
\begin{equation}
  \widehat{\update}_r^{n,k}
  =\Decoder{\systemtype(r)}(\context_r,\boldsymbol\xi_r^{n,k}).
  \label{eq:finite_step_update}
\end{equation}
The same symbol $\widehat{\update}_r^{n,k}$ represents the different physical
quantities listed above. The following section specifies their interpretation,
reshaping, and conversion into the next state.

\subsection{System-Specific Numerical Execution}
\label{sec:numerical_updates}

Let $\NumericalMap{\systemtype(r)}^{\Delta t_r}$ be the PDE family's fixed
numerical update over $\Delta t_r$. It combines the current state $\state_r^n$
with the neural predictions from \cref{eq:finite_step_update} to produce the
next-state estimate $\widehat{\state}_r^{n+1}$:
\begin{equation}
  \widehat{\state}_r^{n+1}
  =\NumericalMap{\systemtype(r)}^{\Delta t_r}
  \!\left(\state_r^n,\{\widehat{\update}_r^{n,k}\}_{k=1}^{K_r}\right).
  \label{eq:execution}
\end{equation}
The program has no trainable parameters. For multi-stage updates, it constructs
successive query states from earlier predictions and reuses the history context
and readout weights. Fixed input/output mappings connect its grid to the observed
state layout; the PDE's physical coefficients are implicit in its arguments.
Below we give the core updates, omitting setting index $r$. Let $N$ be the
program's grid-cell count, $\mathbf1_N$ the all-ones vector, and
$\langle\cdot\rangle$ spatial mean. A superscript $+$ denotes a program's
updated state. Operator details and update properties are in
Appendices~\ref{app:expert_view}--\ref{app:theory}.

\paragraph{Incompressible Navier--Stokes (NS).}
Here the neural output $\widehat{\update}^{n,k}$ is an effective vorticity
transport rate. Let $\boldsymbol\omega\in\R^N$ be the current substep vorticity,
$\mathbf f\in\R^N$ the forcing, $\nu$ the viscosity, and $\delta t$ the substep
duration. Write $f_{\mathrm{diff},\nu}^{s}$ for numerical diffusion over duration
$s$. Let $\boldsymbol\omega_a$ be the diffused state and
$\boldsymbol\omega_m$ the midpoint state; $\widehat{\mathbf r}_1$ and
$\widehat{\mathbf r}_2$ are the corresponding rate outputs of
\cref{eq:finite_step_update} at these two states.
The updated vorticity $\boldsymbol\omega^+$ is computed in this order:
\begin{equation}
\begin{aligned}
 \boldsymbol\omega_a&=f_{\mathrm{diff},\nu}^{\delta t/2}(\boldsymbol\omega),
 &\boldsymbol\omega_m&=\boldsymbol\omega_a+\tfrac{\delta t}{2}
        (\widehat{\mathbf r}_1+\mathbf f),\\
 \boldsymbol\omega^+&=f_{\mathrm{diff},\nu}^{\delta t/2}
       \!\left(\boldsymbol\omega_a+\delta t(\widehat{\mathbf r}_2+\mathbf f)\right).
\end{aligned}
\label{eq:ns_core_update}
\end{equation}
Composing these substeps yields the next saved vorticity state. Transport rates
are learned; diffusion and midpoint integration are fixed operations.

\paragraph{Compressible Navier--Stokes (CNS).}
Reshape the neural output $\widehat{\update}^{n,1}$ as a normalized action matrix
$\mathbf A\in\R^{N\times4}$ for density, two momentum components, and total
energy. Let $\mathbf Q\in\R^{N\times4}$ be the current conservative state,
$\bar{\mathbf a}\in\R^4$ the action's channel means, and
$f_{\mathrm{scale}}(\mathbf Q)\in\R^4$ its conservative-field channel scales;
$\diag$ forms a diagonal matrix. Denote the increment by $\Delta\mathbf Q$,
the candidate by $\mathbf Q^\star$, and the updated state by $\mathbf Q^+$.
A global coefficient $\alpha_{\mathrm{CNS}}\in[0,1]$ enforces positive density
and pressure:
\begin{equation}
\begin{aligned}
 \Delta\mathbf Q&=(\mathbf A-\mathbf1_N\bar{\mathbf a}^{\top})
       \diag(f_{\mathrm{scale}}(\mathbf Q)),\qquad
 \mathbf Q^\star=\mathbf Q+\Delta\mathbf Q,\\
 \mathbf Q^+&=\mathbf Q+\alpha_{\mathrm{CNS}}(\mathbf Q^\star-\mathbf Q).
\end{aligned}
\label{eq:cns_core_update}
\end{equation}
The program restores candidate conservative means before the segment update
and converts $\mathbf Q^+$ back to primitive variables for the next state.
For admissible inputs, global mass, momentum, and energy are preserved in
exact arithmetic (\cref{prop:cns}).

\paragraph{Shallow-water equation (SWE).}
The neural output $\widehat{\update}^{n,1}=\Delta\mathbf h\in\R^N$ is a height
increment. Let $\mathbf h^n$ and $\widehat{\mathbf h}^{n+1}$ be the current and
predicted height vectors, $\mathbf d$ the centered increment, and
$\alpha\in[0,1]$ a global coefficient enforcing a positive height floor.
The update is
\begin{equation}
 \mathbf d=\Delta\mathbf h-\langle\Delta\mathbf h\rangle\mathbf1_N,
 \qquad
 \widehat{\mathbf h}^{n+1}=\mathbf h^n+\alpha\mathbf d.
 \label{eq:swe_centered_update}
\end{equation}
Centering and uniform scaling preserve mass on the uniform grid in exact
arithmetic while respecting the height floor for admissible inputs
(\cref{prop:swe}).

\paragraph{Diffusion--reaction (DR).}
The neural outputs $\boldsymbol\delta_1=\widehat{\update}^{n,1}$ and
$\boldsymbol\delta_2=\widehat{\update}^{n,2}$ are concentration-increment
proposals in $\R^{2N}$ for the two fields. The first is queried at the current
state $\state^n$; the second at the intermediate state $\boldsymbol\xi^{n,2}$.
The program constructs this state and then adds the averaged proposals to
obtain $\widehat{\state}^{n+1}$:
\begin{equation}
 \boldsymbol\xi^{n,2}=\state^n+\tfrac12\boldsymbol\delta_1,
 \qquad
 \widehat{\state}^{n+1}=\state^n+\tfrac12(\boldsymbol\delta_1+\boldsymbol\delta_2).
 \label{eq:dr_refinement}
\end{equation}

\paragraph{CFDBench (CFDB).}
For the velocity-prediction tasks in CFDB, reshape the neural output
$\widehat{\update}^{n,1}$ as $\widehat{\mathbf V}\in\R^{N\times2}$, containing
the two predicted velocity fields. The next velocity state $\mathbf V^{n+1}$
is assigned directly:
\begin{equation}
 \mathbf V^{n+1}=\widehat{\mathbf V}.
 \label{eq:cfdb_core_update}
\end{equation}
The program retains auxiliary support channels; output mapping restores the
observed layout.

\subsection{Joint Training and Autoregressive Prediction}

Let $\boldsymbol\pi$ be the nonnegative setting-weight vector, with
$\sum_{r\in\DataConfigs}\pi_r=1$, and let $\mathcal O_r\subseteq
\{1,\ldots,C_r\}$ be the set of scored physical-channel indices.
We sample settings from the categorical distribution
$\operatorname{cat}(\boldsymbol\pi)$. Sampling from $\Data_r$ below means drawing
a history--target window from its trajectories. With denominator stabilizer
$\epsilon=10^{-8}$, the expected next-state loss $\Loss$ supervises the
complete predictor--program pair:
\begin{equation}
 \Loss=\mathbb E_{\substack{r\sim\operatorname{cat}(\boldsymbol\pi)\\(\Hist_r^n,\state_r^{n+1})\sim\Data_r}}
 \frac{1}{|\mathcal O_r|}\sum_{c\in\mathcal O_r}
 \frac{\|\widehat{\state}_{r,c}^{n+1}-\state_{r,c}^{n+1}\|_2}
      {\|\state_{r,c}^{n+1}\|_2+\epsilon}.
 \label{eq:training_objective}
\end{equation}
The norms run over valid spatial cells of each field vector, excluding
padded cells. Scored channels are weighted equally within each setting.
Gradients through executed numerical arithmetic
train the selected expert and shared encoder without intermediate labels.
During rollout, each
prediction enters the history window and replaces its oldest frame.

New-family adaptation retains this objective, reusing the shared encoder
with a newly initialized system-specific predictor and a configured fixed
numerical update program, as detailed in Appendix~\ref{app:wave}.

\section{Experiments}
\label{sec:experiments}

We follow MoE-POT's benchmark \citep{wang2025moepot} for joint pre-training and task-specific
fine-tuning, complemented by ablations, within-family generalization,
and transfer to a new PDE family.

\paragraph{Datasets and evaluation.}
Six settings cover five PDE families: NS at viscosities
$\nu=10^{-5}$ and $\nu=10^{-3}$ \citep{li2021fno},
PDEBench CNS, SWE, and DR \citep{takamoto2022pdebench},
and CFDBench velocity prediction \citep{luo2023cfdbench}.
We retain the benchmark's ten-frame inputs, padding, and masking,
and evaluate full-horizon autoregressive predictions using mean relative
$L_2$ error (L2RE). Joint pre-training and task-specific fine-tuning
run for 1000 and 200 epochs, respectively, with final checkpoints
used for evaluation.

\paragraph{Baselines and model scales.}
We compare eight baseline methods across fifteen configurations:
six separately trained models FNO, UNet, FFNO, GK-T, Oformer,
and GNOT and three scales each of pre-trained FNO, DPOT, and MoE-POT.
We also report fine-tuned DPOT and MoE-POT results.
Baseline results and activated-parameter counts are taken from
MoE-POT Table~1.
HotPOT-T/S use approximately 17M/90M activated parameters,
matching MoE-POT-T/S.
Appendices~\ref{app:optimization} and~\ref{app:efficiency}
provide detailed evaluation, training, and resource settings.

\subsection{Main Experiments}
\label{sec:main_experiments}

{\setlength{\intextsep}{6pt}\setlength{\abovecaptionskip}{6pt}
\begin{table}[H]
\centering
\footnotesize
\setlength{\tabcolsep}{3.8pt}
\renewcommand{\arraystretch}{1.00}
\caption{Full-horizon mean L2RE on six PDE settings ($\downarrow$).
CNS lists Mach number $M$ and shear viscosity $\eta$, with bulk viscosity
$\zeta=\eta$. Bold: best within each training block; blue: best overall.
HotPOT uses final E1000/E200 checkpoints. Baselines follow MoE-POT Table~1
\citep{wang2025moepot}, with matched splits, histories, horizons, resolutions,
and metrics. Dashes denote errors above 20.}
\label{tab:main_results}
\sisetup{detect-weight=true,mode=text,group-digits=false,table-number-alignment=center}
\begin{tabular}{l c|S[table-format=1.4] S[table-format=1.6]|S[table-format=1.5] S[table-format=2.5] S[table-format=1.5]|S[table-format=1.5]}
\toprule
\multirow{2}{*}{Model} & \multirow{2}{*}{\shortstack{Activated\\Params}} & \multicolumn{2}{c|}{NS benchmarks} & \multicolumn{3}{c|}{PDEBench} & {\multirow{2}{*}{CFDBench}} \\
\cmidrule(lr){3-4}\cmidrule(lr){5-7}
& & {NS ($10^{-5}$)} & {NS ($10^{-3}$)} & {CNS$(0.1,0.01)$} & {SWE} & {DR} &  \\
\midrule
\multicolumn{2}{c|}{Small models} & \multicolumn{6}{c}{} \\
FNO & 0.5M & \cellcolor{white} \bfseries 0.156 & \cellcolor{white} 0.0128 & \cellcolor{white} 0.170 & \cellcolor{white} 0.00440 & \cellcolor{white} 0.120 & \cellcolor{white} 0.00761 \\
UNet & 25M & \cellcolor{white} 0.198 & \cellcolor{white} 0.0245 & \cellcolor{white} 0.357 & \cellcolor{white} 0.0521 & \cellcolor{white} 0.0971 & \cellcolor{white} 0.0209 \\
FFNO & 1.2M & \cellcolor{white} 0.161 & \cellcolor{white} 0.0256 & \cellcolor{white} 0.183 & \cellcolor{white} 0.00458 & \cellcolor{white} 0.161 & \cellcolor{white} 0.0990 \\
GK-T & 1.1M & \cellcolor{white} 0.260 & \cellcolor{white} 0.0148 & \cellcolor{white} 0.919 & \cellcolor{white} 0.0453 & \cellcolor{white} \bfseries 0.0120 & \cellcolor{white} 0.419 \\
Oformer & 1.8M & \cellcolor{white} 0.289 & \cellcolor{white} \bfseries 0.00319 & \cellcolor{white} \bfseries 0.161 & \cellcolor{white} 0.00474 & \cellcolor{white} 0.991 & \cellcolor{white} \bfseries 0.00444 \\
GNOT & 2.2M & \cellcolor{white} 0.590 & \cellcolor{white} 0.316 & \cellcolor{white} 0.533 & \cellcolor{white} \bfseries 0.00199 & \cellcolor{white} 0.930 & \cellcolor{white} 0.0216 \\
\midrule
\multicolumn{2}{c|}{Pre-trained} & \multicolumn{6}{c}{} \\
FNO-T & 10M & \cellcolor{white} 0.191 & \cellcolor{white} 0.0245 & \cellcolor{white} 0.0859 & \cellcolor{white} 11.0 & \cellcolor{white} 0.530 & \cellcolor{white} 0.00601 \\
FNO-S & 30M & \cellcolor{white} 0.157 & \cellcolor{white} 0.0225 & \cellcolor{white} 0.357 & \cellcolor{white} 0.184 & \cellcolor{white} 0.385 & \cellcolor{white} 0.00460 \\
FNO-M & 150M & \cellcolor{white} 0.141 & \cellcolor{white} 0.00730 & \cellcolor{white} {--} & \cellcolor{white} 0.0104 & \cellcolor{white} \bfseries 0.0112 & \cellcolor{white} {--} \\
DPOT-T & 7.5M & \cellcolor{white} 0.107 & \cellcolor{white} 0.0155 & \cellcolor{white} 0.0168 & \cellcolor{white} 0.00631 & \cellcolor{white} 0.0577 & \cellcolor{white} 0.00673 \\
DPOT-S & 30.8M & \cellcolor{white} 0.0688 & \cellcolor{white} 0.00781 & \cellcolor{white} 0.0244 & \cellcolor{white} 0.00392 & \cellcolor{white} 0.0367 & \cellcolor{white} 0.00870 \\
DPOT-M & 122M & \cellcolor{white} 0.0569 & \cellcolor{white} 0.00708 & \cellcolor{white} 0.0224 & \cellcolor{white} 0.00247 & \cellcolor{white} 0.0288 & \cellcolor{white} 0.0113 \\
MoE-POT-T & 17M & \cellcolor{white} 0.0682 & \cellcolor{white} 0.00768 & \cellcolor{white} 0.0105 & \cellcolor{white} 0.00640 & \cellcolor{white} 0.0411 & \cellcolor{white} 0.00529 \\
MoE-POT-S & 90M & \cellcolor{white} 0.0552 & \cellcolor{white} 0.00583 & \cellcolor{white} 0.00959 & \cellcolor{white} 0.00289 & \cellcolor{white} 0.0342 & \cellcolor{white} 0.00448 \\
MoE-POT-M & 288M & \cellcolor{white} 0.0528 & \cellcolor{white} 0.00570 & \cellcolor{white} 0.00914 & \cellcolor{white} 0.00299 & \cellcolor{white} 0.0300 & \cellcolor{white} 0.00513 \\
\textbf{HotPOT-T} & 17M & \cellcolor{white} 0.0315 & \cellcolor{white} 0.00325 & \cellcolor{white} 0.00671 & \cellcolor{white} 0.00177 & \cellcolor{white} 0.0150 & \cellcolor{white} \bfseries 0.00219 \\
\textbf{HotPOT-S} & 90M & \cellcolor{white} \bfseries 0.0299 & \cellcolor{white} \bfseries 0.00323 & \cellcolor{white} \bfseries 0.00623 & \cellcolor{white} \bfseries 0.00166 & \cellcolor{white} 0.0326 & \cellcolor{white} 0.00225 \\
\midrule
\multicolumn{2}{c|}{Fine-tuned} & \multicolumn{6}{c}{} \\
DPOT-T & 7.5M & \cellcolor{white} 0.0700 & \cellcolor{white} 0.00725 & \cellcolor{white} 0.0168 & \cellcolor{white} 0.00313 & \cellcolor{white} 0.0289 & \cellcolor{white} 0.00391 \\
DPOT-S & 30.8M & \cellcolor{white} 0.0502 & \cellcolor{white} 0.00635 & \cellcolor{white} 0.0238 & \cellcolor{white} 0.00315 & \cellcolor{white} 0.0215 & \cellcolor{white} 0.00586 \\
DPOT-M & 122M & \cellcolor{white} 0.0424 & \cellcolor{white} 0.00593 & \cellcolor{white} 0.0221 & \cellcolor{white} 0.00260 & \cellcolor{white} 0.0175 & \cellcolor{white} 0.00653 \\
MoE-POT-T & 17M & \cellcolor{white} 0.0456 & \cellcolor{white} 0.00493 & \cellcolor{white} 0.00746 & \cellcolor{white} 0.00305 & \cellcolor{white} 0.0188 & \cellcolor{white} 0.00330 \\
MoE-POT-S & 90M & \cellcolor{white} 0.0361 & \cellcolor{white} 0.00376 & \cellcolor{white} 0.00742 & \cellcolor{white} 0.00251 & \cellcolor{white} 0.0182 & \cellcolor{white} 0.00313 \\
MoE-POT-M & 288M & \cellcolor{white} 0.0351 & \cellcolor{white} 0.00388 & \cellcolor{white} 0.00744 & \cellcolor{white} 0.00193 & \cellcolor{white} 0.0140 & \cellcolor{white} 0.00398 \\
\textbf{HotPOT-T} & 17M & \cellcolor{white} 0.0215 & \cellcolor{bestblue!16} \bfseries 0.000481 & \cellcolor{white} 0.00583 & \cellcolor{white} 0.00162 & \cellcolor{bestblue!16} \bfseries 0.00163 & \cellcolor{white} 0.00208 \\
\textbf{HotPOT-S} & 90M & \cellcolor{bestblue!16} \bfseries 0.0193 & \cellcolor{white} 0.000483 & \cellcolor{bestblue!16} \bfseries 0.00536 & \cellcolor{bestblue!16} \bfseries 0.00154 & \cellcolor{white} 0.00436 & \cellcolor{bestblue!16} \bfseries 0.00202 \\
\bottomrule
\end{tabular}
\end{table}
}

\paragraph{Pre-training and fine-tuning.}
Across six settings, two activated-parameter scales, and both training stages,
HotPOT outperforms MoE-POT in all 24 matched-scale comparisons
(\cref{tab:main_results}). Mean relative error reductions are 47.1\% after
pre-training and 54.3\% after fine-tuning, or 50.7\% across both stages.
For example, after fine-tuning, HotPOT-T reduces L2RE relative to MoE-POT-T
by 90.3\% on NS at $\nu=10^{-3}$ and 91.3\% on DR.
Appendices~\ref{app:qualitative}--\ref{app:spacetime} provide paired field
visualizations and space--time error plots. 

\paragraph{Accuracy with fewer activated parameters.}
The 17M HotPOT-T model also outperforms the 90M and 288M MoE-POT models
on every setting in both stages. Across the twelve setting--stage
comparisons, its mean relative error reductions are 47.7\% and 46.5\%,
respectively, while using approximately 81.1\% and 94.1\% fewer
activated parameters.

\paragraph{Parameter accounting.}
Each HotPOT-T active path uses 16.77--17.18M parameters, comprising the 2.012M
shared encoder and the selected system-specific predictor.
Detailed parameter accounting and numerical query counts are provided in
Appendix~\ref{app:efficiency}.

\subsection{Ablation Studies}
\label{sec:ablations}

Full denotes HotPOT-T. All models use 17M configurations and fixed E1000
checkpoints; Mean equally weights six full-horizon L2REs.

\paragraph{Ablations of cross-PDE sharing.}
We test whether to share history encoding and where sharing belongs
relative to private coupling. \textit{Learned private history} gives each
PDE family a separate trainable encoder, retaining expert predictors and
numerical updates, to test history sharing.
\textit{Late-shared evolution} moves shared processing after private
field coupling to test this sharing boundary.

In \cref{tab:ablation}, Full lowers mean error by 3.2\% relative to
private history while reducing stored parameters by 7.6\%.
Private history performs better on SWE and slightly better on DR.
Late sharing raises mean error from 0.0101 to 0.0286, worsening
five settings. The placement controls in Appendix~\ref{app:transfer}
favor history processing before private coupling over late evolution processing;
sharing at the history stage further improves mean accuracy and reduces storage.
{\setlength{\intextsep}{6pt}
\begin{table}[H]
\centering\footnotesize
\setlength{\tabcolsep}{3.5pt}
\caption{Cross-PDE sharing ablations. Full-horizon L2RE; bold marks minima before rounding.}
\label{tab:ablation}
\sisetup{detect-weight=true,mode=text,group-digits=false,table-number-alignment=center}
\begin{tabular}{l S[table-format=1.4] S[table-format=1.5] S[table-format=1.5] S[table-format=1.6] S[table-format=1.4] S[table-format=1.5] S[table-format=1.4]}\toprule
Model & {NS $10^{-5}$} & {NS $10^{-3}$} & {CNS} & {SWE} & {DR} & {CFDB} & {Mean} \\
\midrule
Full HotPOT & \bfseries 0.0315 & \bfseries 0.00325 & \bfseries 0.00671 & 0.00177 & 0.0150 & 0.00219 & \bfseries 0.0101 \\
Learned private history & 0.0322 & 0.00420 & 0.00790 & \bfseries 0.000962 & \bfseries 0.0149 & 0.00222 & 0.0104 \\
Late-shared evolution & 0.0331 & 0.00413 & 0.00941 & 0.00421 & 0.119 & \bfseries 0.00217 & 0.0286 \\
\bottomrule\end{tabular}
\end{table}}

\paragraph{Ablations of numerical experts.}
We test the predictor--update interface across PDE families and, specifically
for NS, whether its numerical update provides a benefit beyond repeated
neural refinement.
\textit{Direct-state predictor} (Direct) replaces predicted quantities and numerical
execution with residual state prediction, retaining the neural architecture,
inputs, masks, and nominal readout-call counts. This changes the interface
and stage-state queries together.
\textit{Generic iteration} targets the NS expert, replacing its numerical
update with repeated, state-dependent residual predictions using shared
readout weights. Other expert paths retain their updates.

In \cref{tab:numerical_ablation}, Full lowers mean error from 0.0324
to 0.0101 relative to direct prediction, improving five settings;
Direct is slightly better on CFDB. Relative to Generic iteration, Full lowers
NS error by 1.5\% at $\nu=10^{-5}$ and 19.7\% at $\nu=10^{-3}$.
These comparisons support the paired predictor--update interface for aggregate
accuracy and NS numerical execution over generic learned iteration on both
viscosity settings.
Appendix~\ref{app:transfer} details the controls and temporal analyses;
Appendix~\ref{app:physics} reports physical diagnostics.
{\setlength{\intextsep}{6pt}
\begin{table}[H]
\centering\footnotesize
\setlength{\tabcolsep}{3.5pt}
\caption{Numerical-expert ablations. Full-horizon L2RE; bold marks minima before rounding.}
\label{tab:numerical_ablation}
\sisetup{detect-weight=true,mode=text,group-digits=false,table-number-alignment=center}
\begin{tabular}{l S[table-format=1.4] S[table-format=1.5] S[table-format=1.5] S[table-format=1.5] S[table-format=1.4] S[table-format=1.5] S[table-format=1.4]}\toprule
Model & {NS $10^{-5}$} & {NS $10^{-3}$} & {CNS} & {SWE} & {DR} & {CFDB} & {Mean} \\
\midrule
Full HotPOT & \bfseries 0.0315 & \bfseries 0.00325 & \bfseries 0.00671 & \bfseries 0.00177 & \bfseries 0.0150 & 0.00219 & \bfseries 0.0101 \\
Direct-state predictor & 0.0525 & 0.00625 & 0.0104 & 0.00190 & 0.122 & \bfseries 0.00211 & 0.0324 \\
Generic iteration & 0.0320 & 0.00404 & 0.00720 & 0.00177 & 0.0150 & 0.00234 & 0.0104 \\
\bottomrule\end{tabular}
\end{table}}

\subsection{Within-Family Generalization}
\label{sec:generalization}

We evaluate pre-trained 17M models at the original resolutions under unseen
conditions within supported PDE families.

\paragraph{Zero-shot prediction.}
For CNS, $M$ is the Mach number and $\eta$ the shear viscosity.
The bulk viscosity is $\zeta=\eta$ in the reported datasets
\citep{takamoto2022pdebench}.
Fixed pre-training checkpoints predict CNS at $(M,\eta)=(1,10^{-2})$
and CFDB boundary-condition, geometry, and physical-property subsets
within the cavity and cylinder families, without target training or
checkpoint selection. We report L2RE over ten forecast steps, denoted H10,
with case-uniform masked aggregation for CFDB. This differs from the
full-horizon main benchmark. \Cref{tab:zero_shot} shows that HotPOT lowers
error relative to MoE-POT in all seven groups, including a 55.5\% CNS reduction.

{\setlength{\intextsep}{6pt}
\begin{table}[H]
\centering\footnotesize
\setlength{\tabcolsep}{5pt}
\caption{Zero-shot H10 L2RE ($\downarrow$) using fixed 17M pre-training
checkpoints without target adaptation. CFDB uses case-uniform masked aggregation.}
\label{tab:zero_shot}
\sisetup{detect-weight=true,mode=text,group-digits=false,table-number-alignment=center}
\begin{tabular}{l S[table-format=1.3] S[table-format=1.5] S[table-format=1.4] S[table-format=1.5] S[table-format=1.5] S[table-format=1.5] S[table-format=1.6]}\toprule
 & {CNS} & \multicolumn{3}{c}{Cavity} & \multicolumn{3}{c}{Cylinder} \\
\cmidrule(lr){3-5}\cmidrule(lr){6-8}
Model & {$(M,\eta)=(1,10^{-2})$} & {BC} & {Geo$^\dagger$} & {Prop} & {BC} & {Geo$^\dagger$} & {Prop} \\
\midrule
HotPOT-T & \bfseries 0.189 & \bfseries 0.00568 & \bfseries 0.0169 & \bfseries 0.00467 & \bfseries 0.00119 & \bfseries 0.00180 & \bfseries 0.000577 \\
MoE-POT-T & 0.423 & 0.0276 & 0.0439 & 0.0116 & 0.00802 & 0.0117 & 0.00223 \\
\bottomrule\end{tabular}
\par\vspace{3pt}
\parbox{\linewidth}{\footnotesize BC, Geo, and Prop denote boundary-condition,
geometry, and physical-property subsets. $^\dagger$Each geometry subset contains
one held-out case within its respective cavity or cylinder family.}
\end{table}}

\paragraph{Few-shot adaptation.}
Pre-trained models adapt to NS ($\nu=10^{-4}$) and CNS
($(M,\eta)=(1,10^{-2})$) using 16, 64, or 256 training trajectories.
On 200 test trajectories per target, HotPOT achieves lower full-rollout L2RE
in all six comparisons, with error decreasing across budgets (\cref{tab:fewshot}).
Rollout validation selects checkpoints; Appendix~\ref{app:fewshot_evaluation}
details the protocol.

{\setlength{\intextsep}{6pt}
\begin{table}[H]
\centering\footnotesize
\setlength{\tabcolsep}{6pt}
\caption{Within-family adaptation: model-feedback L2RE on 200 test trajectories
per target (NS/CNS: 20/11 steps). Full-rollout validation selects checkpoints
during 500-epoch adaptation.}
\label{tab:fewshot}
\sisetup{detect-weight=true,mode=text,group-digits=false,table-number-alignment=center}
\begin{tabular}{l S[table-format=1.3] S[table-format=1.4] S[table-format=1.4] S[table-format=1.4] S[table-format=1.4] S[table-format=1.4]}\toprule
 & \multicolumn{3}{c}{NS $\nu=10^{-4}$} & \multicolumn{3}{c}{CNS $(M,\eta)=(1,10^{-2})$} \\
\cmidrule(lr){2-4}\cmidrule(lr){5-7}
Model / shots & {16} & {64} & {256} & {16} & {64} & {256} \\
\midrule
HotPOT-T & \bfseries 0.124 & \bfseries 0.0964 & \bfseries 0.0819 & \bfseries 0.0983 & \bfseries 0.0599 & \bfseries 0.0365 \\
MoE-POT-T & 0.304 & 0.206 & 0.108 & 0.365 & 0.153 & 0.0735 \\
\bottomrule\end{tabular}
\end{table}}

\subsection{Transfer to a New PDE Family}
\label{sec:new_family}

Wave-Layer and Wave-Gauss \citep{herde2024poseidon} belong to one wave-equation
family absent from both pre-training mixtures. HotPOT jointly trains its transferred
encoder and a new scalar predictor. Scratch uses the same architecture, initial
predictor weights, residual updates, and training protocol, but a random encoder.
This comparison tests whether shared encoder pre-training benefits adaptation
to a new PDE family.

At 16, 32, 64, and 128 shots, rollout validation selects checkpoints evaluated on 200
held-out trajectories per dataset using mean per-step L2RE. Encoder transfer
outperforms both baselines in all eight comparisons. Error reductions are
11.9--39.8\% relative to MoE-POT full fine-tuning; relative to Scratch, they are
6.5--8.4\% on Wave-Layer and 0.6--16.4\% on Wave-Gauss, supporting transfer of the pre-trained history
representation. Appendix~\ref{app:wave} gives results (\cref{tab:wave_transfer}),
protocol, and splits.

\section{Discussion and Conclusion}
\label{sec:conclusion}

HotPOT combines shared field-history encoding with PDE-specific
neural--numerical experts, allowing different equations to use different
prediction quantities and fixed update operations. Across six
pre-training settings on regular grids, HotPOT outperforms MoE-POT at
matched activated-parameter scales before and after fine-tuning.
Ablations support the sharing and prediction--execution designs, while
transfer to an unseen wave-equation family supports encoder reuse with
limited target data. System-specific experts increase storage as the
model covers more PDEs, and adapting to a new family requires training
a predictor and configuring its numerical update. Multi-stage updates
can also increase inference cost through repeated neural readouts.
Future work will explore reusable experts, automatic numerical-program
composition, and efficient execution for broader PDE collections.

\FloatBarrier
\label{page:main-end}

\subsection*{AI use statement}
Generative AI tools assisted with language editing.

\subsection*{Reproducibility statement}
The appendix specifies the data interface, model dimensions, system-specific
numerical updates, optimization schedule, parameter accounting, and checkpoint
selection. Code, configurations, and fixed evaluation scripts are available in the
public repository linked in the abstract.

\bibliography{references}
\bibliographystyle{iclr2027_conference}

\clearpage
\appendix
\raggedbottom
\renewcommand{\topfraction}{0.95}
\renewcommand{\bottomfraction}{0.9}
\renewcommand{\textfraction}{0.05}
\renewcommand{\floatpagefraction}{0.95}
\setcounter{topnumber}{4}
\setcounter{bottomnumber}{3}
\setcounter{totalnumber}{6}
\makeatletter
\setlength{\@fptop}{0pt}
\makeatother
\section{Implementation and Additional Results}
\label{app:details}

\subsection{Backbone implementation}
\label{app:backbone}

The patch encoding, temporal aggregation, and grouped Fourier/local blocks
below build on DPOT and MoE-POT \citep{hao2024dpot,wang2025moepot}.

All inputs use ten saved frames at a common $128^2$ analysis resolution.
The deterministic feature map constructs six channels per scalar field:
normalized value, two coordinates, boundary distance, geometry support,
and field validity. Both scales use patch size 8, four adaptive Fourier
neural operator (AFNO) channel groups, and 16 retained modes. The scale-dependent dimensions are listed in
\cref{tab:architecture_dimensions}. The context decoder expands to full
resolution before two $3\times3$ convolutions and a $1\times1$ typed readout.
Per-field history means and standard deviations are computed over the ten
frames and spatial grid, with standard deviations floored at $10^{-6}$.
The system-specific mixers receive these means and log-scales in addition to
shared features and task embeddings; the statistics have no learned
parameters.
Coordinates
span $[-1,1]$; boundary distance is constant on periodic grids and the
minimum coordinate distance to an edge otherwise. These are input-feature
conventions, separate from the expert's numerical update program.

\begin{table}[H]
\centering\small
\caption{Backbone and system-specific predictor dimensions. The neural template
and numerical programs are common to both scales.}
\label{tab:architecture_dimensions}
\begin{tabular}{lrr}\toprule
Component & HotPOT-T (17M active) & HotPOT-S (90M active) \\
\midrule
Shared width / depth / FFN width & 256 / 2 / 512 & 512 / 3 / 2048 \\
System-specific width / depth / FFN width & 512 / 5 / 2048 & 1024 / 7 / 3968 \\
Decoded context channels & 32 & 64 \\
Typed-head hidden width & 128 & 256 \\
Task-feature channels / hidden width & 16 / 64 & 32 / 128 \\
Task-embedding hidden width & 128 & 128 \\
\bottomrule\end{tabular}
\end{table}

\paragraph{Input encoding and positional embedding.}

For a grid of $N_x\times N_y$ cells, flattening only the spatial axes gives
the feature matrix $\mathbf X^j\in\R^{N_xN_y\times d_{\mathrm{in}}}$
at history position $j\in\{1,\ldots,T\}$, with $d_{\mathrm{in}}=6$.
The implementation embeds these six features in an 11-slot tensor with
five fixed zero slots before the patch convolution. This fixed embedding
is included in $f_{\mathrm{patch}}$; parameter counts include the stored
convolution weights.
A convolutional patch encoder $f_{\mathrm{patch}}$ maps non-overlapping
$p\times p$ patches to $d$-dimensional tokens, where $p=8$ and
$d=d_{\mathrm{enc}}$ in this encoder. With
$P=(N_x/p)(N_y/p)$ and learned positional matrix
$\mathbf E_{\mathrm{pos}}\in\R^{P\times d}$,
\begin{equation}
 \mathbf Z^j=f_{\mathrm{patch}}(\mathbf X^j)+\mathbf E_{\mathrm{pos}},
 \qquad \mathbf Z^j\in\R^{P\times d}.
 \label{eq:patch_encoding}
\end{equation}
The patch encoder consists of a strided convolution, GELU
\citep{hendrycks2016gelu}, and a pointwise projection. Its parameters,
including $\mathbf E_{\mathrm{pos}}$, are part of $\boldsymbol\theta$
and are shared across history frames, scalar fields, and datasets.
Spatial convolutions and FFTs reshape these matrices to the two-dimensional
grid and flatten them back in the same order.

\paragraph{Temporal aggregation.}
With field, setting, and saved-time indices suppressed, let
$\mathbf z_i^j=(\mathbf Z^j_{i,:})^\top\in\R^d$ be the token vector at
patch $i$. The matrix $\mathbf Z_i^{\mathrm{hist}}\in\R^{T\times d}$ in
\cref{fig:architecture} stacks $\{(\mathbf z_i^j)^\top\}_{j=1}^T$.
A real-valued adaptation of Fourier-feature aggregation uses learned
channel-mixing matrices
$\mathbf W_j\in\R^{d\times d}$ and frequency vector
$\boldsymbol\kappa\in\R_{>0}^d$. Let $\lambda_j=(j-T)/(T-1)$ denote
the normalized temporal lag, $\odot$ elementwise multiplication, and
$\mathbf v_i^0\in\R^d$ the aggregated token. The cosine acts componentwise:
\begin{equation}
 \mathbf v_i^0=\sum_{j=1}^{T}\mathbf W_j^\top
 \left[\mathbf z_i^j\odot\cos(\boldsymbol\kappa\lambda_j)\right].
 \label{eq:temporal_aggregation}
\end{equation}
Let $\mathbf V^0\in\R^{P\times d}$ denote the matrix formed by stacking
$\{(\mathbf v_i^0)^\top\}_{i=1}^P$.

\paragraph{Fourier and local mixing blocks.}
Spatial processing uses the grouped spectral mixer inherited from the
MoE-POT implementation, following the AFNO/Fourier operator-transformer
lineage \citep{guibas2022afno}. The mixer applies
a two-dimensional fast Fourier transform (FFT), splits feature channels
into groups, transforms each group with a two-layer frequency-domain
multilayer perceptron (MLP), and returns to the spatial grid by inverse FFT.
The MLP parameters are shared over retained frequency locations; GELU acts
separately on real and imaginary components, with no soft-thresholding.
Let $f_{\mathrm{fft}}$ and $f_{\mathrm{ifft}}$ denote the spatial FFT
and its inverse, and let $f_{\mathrm{freq},b}$ be the grouped MLP with
unretained modes set to zero. For block index $b$ and input
$\mathbf U\in\R^{P\times d}$, including
the inherited internal skip connection gives
\begin{equation}
 f_{\mathrm{spec},b}(\mathbf U)=\mathbf U+f_{\mathrm{ifft}}
 \left(f_{\mathrm{freq},b}(f_{\mathrm{fft}}(\mathbf U))\right).
 \label{eq:fourier_module}
\end{equation}
Each block combines this spectral path
with a local convolutional function $f_{\mathrm{local},b}$ and a channel
feed-forward function $f_{\mathrm{ffn},b}$. The functions
$f_{\mathrm{gn},b,1}$ and $f_{\mathrm{gn},b,2}$ are the block's two group
normalizations \citep{wu2018groupnorm}. Let $\mathbf V^b$, $\mathbf U^b$,
$\widetilde{\mathbf V}^b$, and $\mathbf V^{b+1}\in\R^{P\times d}$ denote
the block input, normalized input, mixed features, and output, respectively:
\begin{align}
 \mathbf U^b&=f_{\mathrm{gn},b,1}(\mathbf V^b),\nonumber\\
 \widetilde{\mathbf V}^b
   &=\mathbf V^b+f_{\mathrm{spec},b}(\mathbf U^b)
     +f_{\mathrm{local},b}(\mathbf U^b),\nonumber\\
 \mathbf V^{b+1}
   &=\widetilde{\mathbf V}^b+
     f_{\mathrm{ffn},b}\!\left(f_{\mathrm{gn},b,2}
       (\widetilde{\mathbf V}^b)\right).
 \label{eq:backbone_block}
\end{align}
All spatial feature matrices have shape $P\times d$, where $d$ is the
width of the shared or system-specific block in question. The local path
uses a depthwise $3\times3$ convolution, GELU, and a $1\times1$
convolution; the FFN uses two $1\times1$ convolutions with GELU.
After the shared blocks and a final normalization, the encoder produces
$\latent_{r,c}$; their parameters belong to $\boldsymbol\theta$.
The corresponding system-specific blocks use $\boldsymbol\phi_{\systemtype(r)}$.

\subsection{System-specific numerical update programs}
\label{app:expert_view}

Each expert instantiates \cref{eq:execution} by pairing a system-specific
neural predictor with a numerical update program. Field-wise history encoding,
system-specific state evolution, and spatial context decoding run once per
saved interval. The program then constructs query states and reuses the
predictor's readout. All predicted
quantities are learned through next-state supervision; they require no
separate rate, action, or increment labels. The programs below detail
the five alternatives illustrated in \cref{fig:architecture}. We suppress
the setting index $r$ below; $N$ denotes the grid-cell count and a
superscript $+$ the updated state. Vectors collect spatial entries;
multi-channel arrays are written as matrices when channels form columns.
The all-ones vector is $\mathbf 1_N\in\R^N$, and
$\langle\mathbf a\rangle=N^{-1}\mathbf 1_N^\top\mathbf a$ is the
spatial mean of a field vector $\mathbf a$.

\subsubsection{NS: transport-rate integration}
\label{app:ns_program}

The neural transport-rate query in \cref{eq:finite_step_update} uses stage
vorticity, history context, and deterministic Biot--Savart transport and forcing
features. Let $f_{\mathrm{rate}}:\R^N\to\R^N$ denote this query with the
decoded history context and physical coefficients held fixed. Using the
diffused and midpoint states $\boldsymbol\omega_a,\boldsymbol\omega_m$ from
\cref{eq:ns_core_update}, its two outputs are
\begin{equation}
 \widehat{\mathbf r}_1=f_{\mathrm{rate}}(\boldsymbol\omega_a),\qquad
 \widehat{\mathbf r}_2=f_{\mathrm{rate}}(\boldsymbol\omega_m).
 \label{eq:ns_rate_queries}
\end{equation}
The program adds forcing to each rate and uses the explicit midpoint update
between two diffusion half-steps, forming a Strang split
\citep{strang1968construction}. The same decoded context and readout weights
are reused at both queries and throughout the saved interval.

For the diffusion operator $f_{\mathrm{diff},\nu}^{s}$, let $\delta t_{\mathrm{CN}}>0$ be a
Crank--Nicolson time increment, $s$ the diffusion duration, and
$m=s/\delta t_{\mathrm{CN}}\in\mathbb N$ its number of increments. Let $f_{\mathrm{fft}}$ and
$f_{\mathrm{ifft}}$ be the discrete Fourier transform on the periodic grid and
its inverse. For Fourier mode $\ell$, let $\lambda_\ell\ge0$ be the eigenvalue
of the negative Laplacian. With viscosity $\nu$, define the mode multiplier
$b_\ell$ and the diffusion update of a vorticity vector
$\boldsymbol\omega\in\R^N$ by
\begin{equation}
 b_\ell=\frac{1-\delta t_{\mathrm{CN}}\nu\lambda_\ell/2}{1+\delta t_{\mathrm{CN}}\nu\lambda_\ell/2},\qquad
 f_{\mathrm{diff},\nu}^{s}(\boldsymbol\omega)
 =f_{\mathrm{ifft}}\!\left(
 \{b_\ell^{m}(f_{\mathrm{fft}}(\boldsymbol\omega))_\ell\}_\ell\right).
 \label{eq:ns_diffusion}
\end{equation}
This composes $m$ Crank--Nicolson diffusion increments
\citep{crank1947practical}. The implementation uses a positive multiplier,
with $\delta t_{\mathrm{CN}}\nu\lambda_\ell/2<1$, and evaluates its power in logarithmic form.
Each half-step takes $s=\delta t/2$; composing the split substeps covers the
saved interval $\Delta t_r$.

\subsubsection{CNS: conservative action and admissible update}
\label{app:cns_program}

At cell $i$, let $u_{x,i},u_{y,i}$ denote velocity components,
$\rho_i$ density, and $p_i$ pressure. Let $\gamma>1$ be the specific-heat ratio;
the total energy density is
$E_i=p_i/(\gamma-1)+\rho_i(u_{x,i}^2+u_{y,i}^2)/2$.
The conservative vector is
$\mathbf c_i=(\rho_i,\rho_i u_{x,i},\rho_i u_{y,i},E_i)^\top\in\R^4$.
The full conservative state is the matrix $\mathbf Q\in\R^{N\times4}$
with row $\mathbf c_i^\top$.
One predictor query supplies a normalized action, reshaped as
$\mathbf A\in\R^{N\times4}$. Its channel means are
$\bar{\mathbf a}=N^{-1}\mathbf A^\top\mathbf1_N\in\R^4$.
Let $\epsilon_s>0$ be a numerical scale floor, and let
$f_{\mathrm{scale}}:\R^{N\times4}\to\R^4$ return the per-channel
spatial root mean square (RMS) of centered conservative fields, floored
at $\epsilon_s$; $\diag$ places these scales on a matrix diagonal.
Denote the conservative increment by $\Delta\mathbf Q$ and its candidate
state by $\mathbf Q^\star$. The program constructs
\begin{equation}
 \Delta\mathbf Q=(\mathbf A-\mathbf1_N\bar{\mathbf a}^\top)
     \diag(f_{\mathrm{scale}}(\mathbf Q)),\qquad
 \mathbf Q^\star=\mathbf Q+\Delta\mathbf Q.
 \label{eq:cns_action}
\end{equation}
Multiplication by the diagonal matrix scales each
channel uniformly over the grid. The program restores the candidate
conservative means to those of $\mathbf Q$. Let $\rho_{\min},p_{\min}>0$
be density and pressure floors, and let $\alpha_{\mathrm{CNS}}\in[0,1]$
be one global coefficient ensuring these strict bounds in the updated
state $\mathbf Q^+$. The segment update is
\begin{equation}
 \mathbf Q^+=\mathbf Q+\alpha_{\mathrm{CNS}}(\mathbf Q^\star-\mathbf Q),
 \qquad \alpha_{\mathrm{CNS}}\in[0,1].
 \label{eq:cns_segment_update}
\end{equation}
For an admissible starting state, an admissible candidate uses
$\alpha_{\mathrm{CNS}}=1$; otherwise a prescribed number
$B_{\mathrm{bis}}$ of bisections determines an admissible lower endpoint
of the segment. Boolean tests select the segment, and
autograd differentiates its executed arithmetic. The result is converted
back to primitive variables. The preserved quantities are the global totals
of the four conservative channels on the uniform grid.
\Cref{prop:cns} establishes the admissibility and conservation properties.

\subsubsection{DR: two-query proposal refinement}
\label{app:dr_program}

The sum of two predictor output groups gives a finite-interval concentration
proposal $\boldsymbol\delta_1\in\R^{2N}$ at the current state $\state^n$;
the two concentration fields are concatenated in the same order as $\state^n$.
A second query at an intermediate state gives $\boldsymbol\delta_2$;
the program averages the two proposals as in \cref{eq:dr_refinement}.
Both queries use the same context and readout weights. These outputs are
concentration proposals learned through the resulting next-state error;
the fixed operation is the query-and-average refinement.

\subsubsection{SWE: mass-preserving height increment}
\label{app:swe_program}

Only height is observed, with current field vector $\mathbf h^n\in\R^N$ and
components $h_i^n$. One predictor query directly predicts a height
increment $\Delta\mathbf h\in\R^N$ in the observed field units. The program
centers it and applies one coefficient $\alpha\in[0,1]$ to the entire
grid, as specified in \cref{eq:swe_centered_update}.
Let $h_{\min}>0$ be a prescribed height floor and $\epsilon_d>0$ a
numerical denominator floor. For an initial height satisfying
$h_i^n\ge h_{\min}$, the coefficient in \cref{eq:swe_limiter} ensures
$\widehat h_i^{n+1}\ge h_{\min}$ at every cell $i$.
Since the increment has zero spatial mean and $\alpha$ is uniform,
the update preserves mass in exact arithmetic. The role of the limiter
is to shorten the complete increment, rather than clip individual cells.
\Cref{prop:swe} gives the formula and proof.

\subsubsection{CFDB: next-velocity assignment}
\label{app:cfdb_program}

One predictor query estimates the next two velocity fields, concatenated
as a vector in $\R^{2N}$. The program replaces the current velocity
channels and copies auxiliary support channels unchanged, before
the output mapping restores the evaluation resolution and channel layout.
Within CFDB, spectral/local layers are shared across geometry variants,
while pointwise FFNs and velocity heads are geometry-specific.

\subsection{Properties of the numerical updates}
\label{app:theory}

We establish the update properties used by the SWE and CNS numerical
programs on a uniform grid. All arrays have finitely many finite-valued
entries, arithmetic is exact, and each admissibility factor is a scalar
shared by all spatial cells of one sample.

\begin{proposition}[Mass and height-floor preservation]
\label{prop:swe}
On a grid with $N$ cells and prescribed height floor $h_{\min}>0$,
let $\mathbf h\in\R^N$ satisfy $h_i\ge h_{\min}$, and let
$\mathbf d\in\R^N$ satisfy $\sum_i d_i=0$. Given a denominator floor
$\epsilon_d>0$, denote the cellwise step bounds by $r_i$, their clipped
minimum by $\alpha$, and the updated height by $\mathbf h^+$. Define
\begin{equation}
 r_i=\begin{cases}
 (h_i-h_{\min})/\max(-d_i,\epsilon_d),&d_i<0,\\
 1,&d_i\ge0,
 \end{cases}
 \qquad
 \alpha=\min(1,\min_i r_i),\qquad \mathbf h^+=\mathbf h+\alpha\mathbf d.
 \label{eq:swe_limiter}
\end{equation}
Then $0\le\alpha\le1$, $h_i^+\ge h_{\min}$ for every cell, and
$\sum_i h_i^+=\sum_i h_i$.
\end{proposition}
\begin{proof}
Every ratio is nonnegative. For $d_i\ge0$, the update cannot decrease
$h_i$. For $d_i<0$, the denominator in \cref{eq:swe_limiter} is at least
$-d_i$, so $\alpha(-d_i)\le h_i-h_{\min}$, which gives the height bound.
Since the same $\alpha$ multiplies every cell,
$\sum_i h_i^+=\sum_i h_i+\alpha\sum_i d_i=\sum_i h_i$.
If $\mathbf d=\mathbf0$, then $\alpha=1$ and the state is unchanged. A cell at the floor
with a negative increment instead gives $\alpha=0$.
\end{proof}

Centering the directly predicted height increment gives $\mathbf d$.
The denominator floor can shorten the permitted step for very small
negative increments, but leaves both guarantees intact. The common scale
is essential to this construction: cellwise clipping would generally alter
the zero-sum increment and hence its mass balance.

\begin{proposition}[CNS conservation and admissible segment]
\label{prop:cns}
At one cell, let $\mathbf m=(\rho u_x,\rho u_y)^\top\in\R^2$ denote
momentum, with components $m_x,m_y$, and $\mathbf c=(\rho,m_x,m_y,E)^\top\in\R^4$ the conservative state. For $\gamma>1$ and positive floors
$\rho_{\min},p_{\min}$, let $p(\mathbf c)$ be the pressure recovered
from the conservative state and $\mathcal K$ its admissible set:
\begin{equation}
 \mathcal K=\left\{\mathbf c\in\R^4:
 \rho>\rho_{\min},\
 p(\mathbf c)=(\gamma-1)\left(E-\frac{\|\mathbf m\|_2^2}{2\rho}\right)>p_{\min}\right\}.
 \label{eq:cns_admissible}
\end{equation}
Let $\mathbf c_i$, $\mathbf c_i^\star$, and $\mathbf c_i^+$ be the
column vectors corresponding to row $i$ of $\mathbf Q$, $\mathbf Q^\star$,
and $\mathbf Q^+$, respectively. Suppose every $\mathbf c_i$ lies in
$\mathcal K$, and a finite candidate satisfies
$\sum_i\mathbf c_i^\star=\sum_i\mathbf c_i$.
The implemented segment search returns $\mathbf Q^+=\mathbf Q+\alpha_{\mathrm{CNS}}(\mathbf Q^\star-\mathbf Q)$ with
every $\mathbf c_i^+\in\mathcal K$ and $\sum_i \mathbf c_i^+=\sum_i \mathbf c_i$.
If $\mathbf Q^\star$ is admissible, $\alpha_{\mathrm{CNS}}=1$. Otherwise,
let $t\in[0,1]$ denote a segment fraction and
$\alpha_{\mathrm{CNS},\star}$ the supremum of admissible fractions:
\[
 \alpha_{\mathrm{CNS},\star}=\sup\{t\in[0,1]: \mathbf c_i+t(\mathbf c_i^\star-\mathbf c_i)\in\mathcal K
                    \text{ for every }i\}.
\]
For a prescribed bisection count $B_{\mathrm{bis}}$, the returned lower
endpoint satisfies
$0\le\alpha_{\mathrm{CNS},\star}-\alpha_{\mathrm{CNS}}\le2^{-B_{\mathrm{bis}}}$.
\end{proposition}
\begin{proof}
For two states with densities $\rho_0,\rho_1>0$ and momenta
$\mathbf m_0,\mathbf m_1$, and a fraction $t\in[0,1]$, weighted
Cauchy--Schwarz gives
\[
 \frac{\|(1-t)\mathbf m_0+t\mathbf m_1\|_2^2}{(1-t)\rho_0+t\rho_1}
 \le (1-t)\frac{\|\mathbf m_0\|_2^2}{\rho_0}
       +t\frac{\|\mathbf m_1\|_2^2}{\rho_1}.
\]
Thus $p(\mathbf c)$ is concave where $\rho>0$, and the two strict inequalities
in \cref{eq:cns_admissible} define a convex set. Intersecting the
cellwise admissible segments gives an interval containing zero.
In the rejected-candidate branch, initialize lower and upper endpoints
at $0$ and $1$. At every bisection, an accepted midpoint replaces the
admissible lower endpoint; a rejected midpoint replaces the upper endpoint.
Convexity keeps $\alpha_{\mathrm{CNS},\star}$ bracketed, and after $B_{\mathrm{bis}}$ iterations the
bracket width is $2^{-B_{\mathrm{bis}}}$. The returned lower endpoint is admissible,
including when it remains zero. For either branch,
\[
 \sum_i \mathbf c_i^+=\sum_i \mathbf c_i+\alpha_{\mathrm{CNS}}
                   \left(\sum_i \mathbf c_i^\star-\sum_i \mathbf c_i\right)=\sum_i \mathbf c_i .
\]
\end{proof}

Mean restoration constructs the conservative candidate required by the
proposition. One segment coefficient then acts on all cells and all four
conservative components, respecting the strict floors
$\rho_{\min},p_{\min}$ for a given specific-heat ratio $\gamma>1$. Reapplying either program from its admissible
output preserves the same constraints in exact arithmetic.
The convex admissible set is also used in positivity-preserving
finite-volume and discontinuous Galerkin schemes \citep{zhang2010positivity}.
Here the segment acts on a whole-field candidate and preserves global totals.

\subsection{Optimization and evaluation}
\label{app:optimization}

\begin{table}[H]
\centering\small
\setlength{\tabcolsep}{5pt}
\caption{Main evaluation populations and fine-tuning schedule for the
pre-training tasks at both model scales. Train is the trajectory count used per epoch;
test is the fixed evaluated row count. $C_{\mathrm{pred},r}=|\mathcal O_r|$ counts scored physical
channels, and $L_{\mathrm{roll}}$ counts model-feedback prediction steps after ten input
frames. The learning rate is the maximum during E200 fine-tuning.}
\label{tab:protocol}
\begin{tabular}{lrrrrrr}\toprule
Setting & Native grid & Train & Test & $C_{\mathrm{pred},r}$ & $L_{\mathrm{roll}}$ & Fine-tune LR \\
\midrule
NS $10^{-5}$ & $64^2$ & 1,000 & 200 & 1 & 10 & $10^{-5}$ \\
NS $10^{-3}$ & $64^2$ & 1,000 & 200 & 1 & 20 & $10^{-3}$ \\
CNS & $128^2$ & 9,000 & 200 & 4 & 11 & $10^{-4}$ \\
SWE & $128^2$ & 900 & 60 & 1 & 91 & $10^{-3}$ \\
DR & $128^2$ & 900 & 60 & 2 & 91 & $10^{-3}$ \\
CFDB & $64^2$ & 9,000 & 1,000 & 2 & 10 & $10^{-5}$ \\
\bottomrule\end{tabular}\end{table}

\paragraph{Joint pre-training.}
Both scales train for 1000 epochs, with 545 updates per epoch, using two
devices and local batch 20 (global batch 40). Each trajectory contributes
one ten-frame-to-one-frame training window per epoch. Over a ten-epoch
cycle, sampling visits the initial window once and early, middle, and late
non-initial windows three times each. Batches contain compatible tasks;
the epoch plan retains the benchmark's per-dataset trajectory exposure.
Adam uses $\beta_1=\beta_2=0.9$ and weight decay $10^{-6}$. OneCycleLR
reaches $10^{-3}$ after 20\% of training, with initial and final division
factors both $10^4$. NS (on its native grid) and DR receive norm-scaled
Gaussian input noise with factor $5\times10^{-4}$; CNS, SWE, CFDB, and all
targets remain clean. The loss is the one-step
physical-channel relative error in \cref{eq:training_objective}.
The 9000 CFDB training rows comprise 581 cavity, 8231 cylinder, and 188 tube
examples; all three geometry paths therefore participate in pre-training.

\paragraph{Fine-tuning on the pre-training tasks.}
Starting from the fixed E1000 checkpoint, the shared encoder and selected
system-specific predictor train for 200 complete epochs; other predictors
remain frozen.
Both scales use global batch 40, the same Adam settings and 20\% warmup,
and target-only phase-aware sampling. Maximum learning rates are specified per task in \cref{tab:protocol}; input noise is disabled. Fine-tuning retains full
per-epoch target exposure, giving 5,000 updates for each NS setting,
45,000 for CNS and CFDB, and 4,600 for SWE and DR. Learning rates are fixed from the T-scale held-out E200 rollout-validation
diagnostics and reused at both scales. Final refits use the complete
original training populations and the fixed E200 checkpoint.

\paragraph{Rollout metric and population.}
After the first ten frames, every predicted frame is fed back to the model.
For each trajectory and physical channel, the released relative-$L_2$
metric flattens the spatial grid and the complete reported prediction
window together. Channel errors are then averaged over scored channels,
and trajectory errors equally over evaluated rows; the denominator includes
$10^{-8}$. Masks exclude padded cells and unscored channels. Thus an H10
entry measures the whole ten-frame prediction window. CFDB scores velocity
channels, with geometry/support channels used only as inputs and state
assembly. The fixed main population uses the first 1000 rows (zero-based
indices $0$--$999$) of the 1182-row packed CFDB test file, comprising 110
cavity and 890 cylinder rows. It is
row-weighted; the separate condition study uses case-uniform aggregation.
Main and ablation results use fixed final checkpoints. Reported runs use
one training seed; few-shot and wave-transfer runs use seed 0.

\paragraph{Few-shot selection protocol.}
NS at $\nu=10^{-4}$ supplies 9800 target-training trajectories, and CNS
at $M=1$ supplies 9000. The last 256 trajectories of each training source
form a fixed validation set. The reported training subsets contain 16,
64, or 256 trajectories sampled from the remaining rows.
Paired methods use identical training and validation indices.
Training runs for 500 epochs with batch $\min(20,\text{shots})$, Adam
with $\beta_1=\beta_2=0.9$, weight decay $10^{-6}$, and OneCycleLR
with 20\% warmup and division factors $10^4$.
NS HotPOT uses maximum learning rates $10^{-4}$, $10^{-4}$, and $10^{-5}$
at the three budgets, respectively, selected from the evaluated
$\{10^{-4},10^{-5}\}$ candidates by rollout validation. MoE-POT uses $10^{-4}$ on both targets; CNS HotPOT also uses $10^{-4}$.
Each training trajectory supplies one phase-aware ten-frame-to-next-frame
window per epoch. Every ten epochs, validation evaluates autoregressive
rollout on all 256 validation trajectories over the target test horizon.
The checkpoint with lowest finite validation L2RE is used for testing.
Both models update their shared and selected task parameters during
adaptation. MoE-POT starts from its local E1000 pre-training checkpoint.

\subsubsection{Few-shot rollout evaluation}
\label{app:fewshot_evaluation}

Each target--budget pair uses its validation-selected checkpoint, including
normalization statistics. Evaluation starts from ten
input frames and uses model feedback for 20 steps on NS and 11 on CNS,
over 200 test trajectories per target. Masked per-channel space--time
L2RE is averaged equally over trajectories.

\subsection{Parameter storage and numerical query counts}
\label{app:efficiency}

\Cref{tab:parameters} gives exact parameter counts at both scales;
\cref{tab:resource_comparison} compares activated and stored parameters across models.

\begin{table}[H]
\centering\small
\caption{Exact parameter accounting at both scales. Storage includes all
system-specific expert paths; an active path includes the shared encoder and the selected
system or CFDB geometry path.}
\label{tab:parameters}
\begin{tabular}{lrr}\toprule
Quantity & HotPOT-T & HotPOT-S \\
\midrule
All-PDE shared history encoder & 2,011,648 & 11,901,952 \\
Active NS/SWE path & 16,770,577 & 89,584,673 \\
Active DR path & 16,903,188 & 90,112,036 \\
Active CFDB geometry path & 16,905,234 & 90,116,130 \\
Active CNS path & 17,178,004 & 91,185,956 \\
Stored total & 97,902,720 & 518,480,368 \\
\bottomrule\end{tabular}\end{table}

\begin{table}[H]
\centering\small
\caption{Activated and stored parameters, in millions. MoE-POT counts are
the published nominal values \citep[Table~5]{wang2025moepot}; HotPOT ranges
cover its active routes.}
\label{tab:resource_comparison}
\begin{tabular}{lrr}\toprule
Model & Activated & Stored total \\
\midrule
MoE-POT-T & 17 & 30 \\
HotPOT-T & 16.77--17.18 & 97.90 \\
MoE-POT-S & 90 & 166 \\
HotPOT-S & 89.58--91.19 & 518.48 \\
MoE-POT-M & 288 & 489 \\
\bottomrule\end{tabular}
\end{table}

Repeated use of the same predictor counts its weights once in the activated
parameter total and contributes to runtime according to the number of calls.

Each setting updates the shared encoder and its selected family-specific
predictor. The two NS viscosity settings use the same NS predictor.
For CFDB, the encoder and geometry-shared layers train together with the
selected geometry's pointwise FFNs and velocity head. The family selector
is deterministic and has no learned router.

\paragraph{Numerical query counts.}
The history encoder, system-specific context blocks, and spatial decoder
run once per saved interval. NS uses two readout calls per numerical
substep; DR uses two per saved interval, and CNS, SWE, and each CFDB
geometry path use one. Repeated calls reuse the same readout weights,
so they increase computation without increasing the
number of activated parameters. The present configuration has no matched
wall-clock benchmark; parameter counts and query counts describe different
aspects of its computational cost.

\subsection{Additional ablation details and analyses}
\label{app:transfer}

\paragraph{History sharing and placement.}
Full uses one learned history encoder across PDE families, followed by
private coupling and evolution. Learned private history copies the same
encoder architecture into a separately trainable encoder for each family.
The two NS settings retain the same family assignment. Predictor
architectures, numerical updates, task exposure, and active-path dimensions
are retained; storage increases because the history weights are no longer
shared. Late-shared evolution uses field-history tokenization before private
coupling and moves the shared processing to the evolution stage.
Late-private evolution uses the corresponding placement with family-private weights.
\Cref{tab:sharing_controls} summarizes the four combinations.

\begin{table}[!htb]
\centering\small
\setlength{\tabcolsep}{3.5pt}
\caption{History/evolution placement and sharing controls. The mean equally weights six settings; all results use fixed E1000 checkpoints.}
\label{tab:sharing_controls}
\sisetup{detect-weight=true,mode=text,group-digits=false,table-number-alignment=center}
\begin{tabular}{lll S[table-format=3.2] S[table-format=1.4]}\toprule
Model & Placement & Weights & {Stored (M)} & {Mean L2RE} \\
\midrule
Full & History & Shared & 97.90 & 0.0101 \\
Learned private history & History & Private & 105.95 & 0.0104 \\
Late-shared evolution & Evolution & Shared & 97.90 & 0.0286 \\
Late-private evolution & Evolution & Private & 105.95 & 0.0287 \\
\bottomrule\end{tabular}
\end{table}

\paragraph{Direct-state and iterative controls.}
\label{app:direct_matching}
Direct abbreviates the Direct-state predictor control. Full and Direct
retain the same encoder, system-specific mixers, context
decoders, and readout-head architectures, with 97,902,720 stored parameters
and identical activated counts. Their history features, statistics, task
inputs, and masks match. Direct replaces the prediction--update interface
with residual state prediction, retaining nominal readout-call counts.
It therefore changes prediction quantities, numerical execution, and
stage-state queries together.
Generic iteration uses the same cached history context and shared readout
weights to refine the NS state through state-dependent residual queries.
Its intervention tests the NS numerical update against a learned iterative
alternative. Other paths retain their configured update structures.
All rows are trained under the same E1000 multi-PDE protocol.

\paragraph{Space--time visualization of two-dimensional tasks.}
\label{app:sharing_spacetime}
The SWE and DR tasks have two spatial dimensions, $x$ and $y$.
The two DR concentration fields are denoted by $u$ and $v$.
\Cref{fig:sharing_spacetime} stacks their two-dimensional error fields
over forecast steps to form $(x,y,n)$ space--time volumes, where $n$
indexes the forecast step. The third rendering axis therefore represents
time, not a third spatial dimension.
We compare HotPOT with Late-shared evolution using fixed example 0 and all
91 forecast steps. Two complementary two-dimensional cuts accompany
each volume pair: an $x$--$y$ map at step 91 locates the final spatial
error differences, and an $x$--time map at zero-based grid index $y=64$ shows how the local
differences evolve. These cuts expose local differences without
occlusion by the volume rendering.

\begin{figure}[!htbp]
\centering
\includegraphics[width=\textwidth]{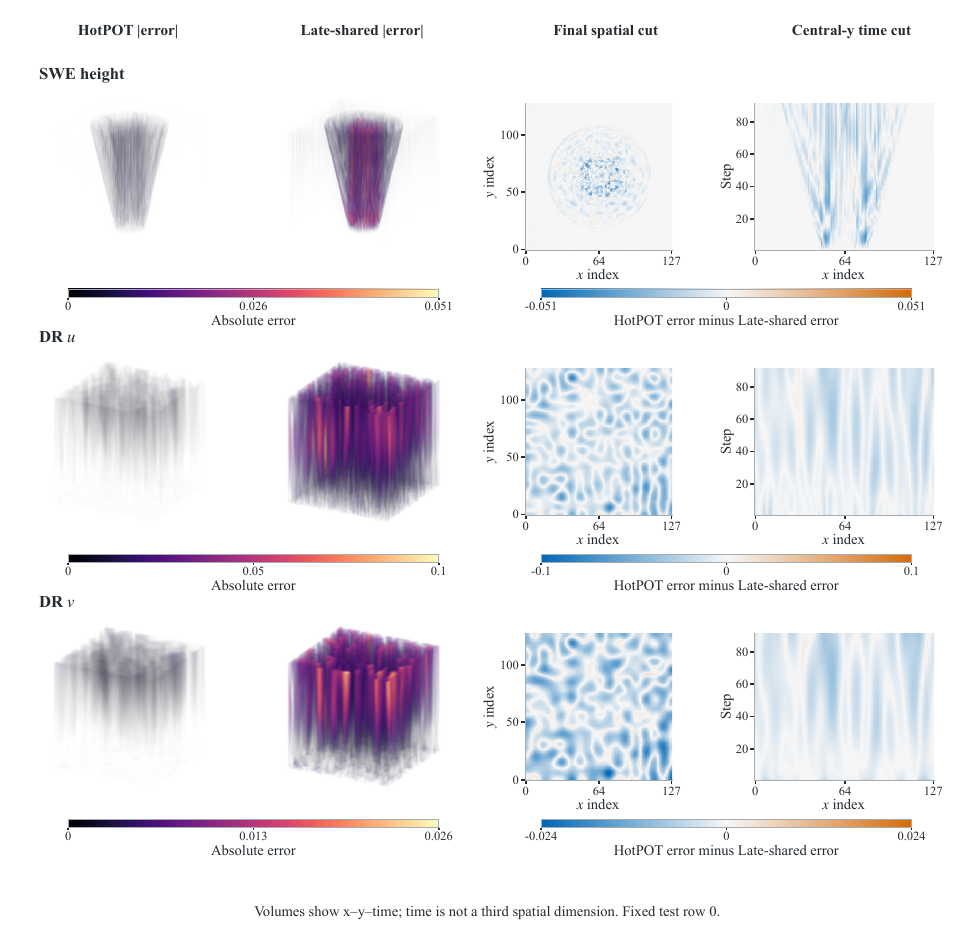}
\caption{\textbf{Space--time errors and two-dimensional cuts for the sharing-location comparison.}
Rows show SWE height and both DR fields for fixed example 0.
The first two columns render absolute errors in two spatial dimensions
and forecast time. The last two columns show HotPOT absolute error minus
Late-shared evolution absolute error: an $x$--$y$ slice at step 91 and an
$x$--time slice at zero-based $y=64$.
Blue denotes lower error for HotPOT; orange denotes lower error for Late-shared evolution.
Within each field, the error volumes share the camera, full-range color
mapping, and opacity mapping; opacity increases with absolute error.
The two cuts share a symmetric range determined by the full space--time
error difference. Volume colorbars describe scalar values before
compositing, whereas the cuts display local scalar differences directly.}
\label{fig:sharing_spacetime}
\end{figure}

In these illustrated trajectories, the DR spatial cuts contain extended
regions of lower HotPOT error, and the time cuts show differences persisting
across many forecast steps. For SWE, the time cut localizes differences
along the moving wave pattern. These local views complement the six-setting
aggregate comparison in \cref{tab:ablation}.

\paragraph{Full-population temporal errors.}
\Cref{fig:all_test_rollout} aggregates every trajectory in the six
main-test populations. It connects the final aggregate scores in
\cref{tab:ablation,tab:numerical_ablation} to the evolution of prediction error through time.
The per-step metric differs from full-horizon L2RE, which forms one norm
over the complete space--time window.

\begin{figure}[!htbp]
\centering
\includegraphics[width=\textwidth]{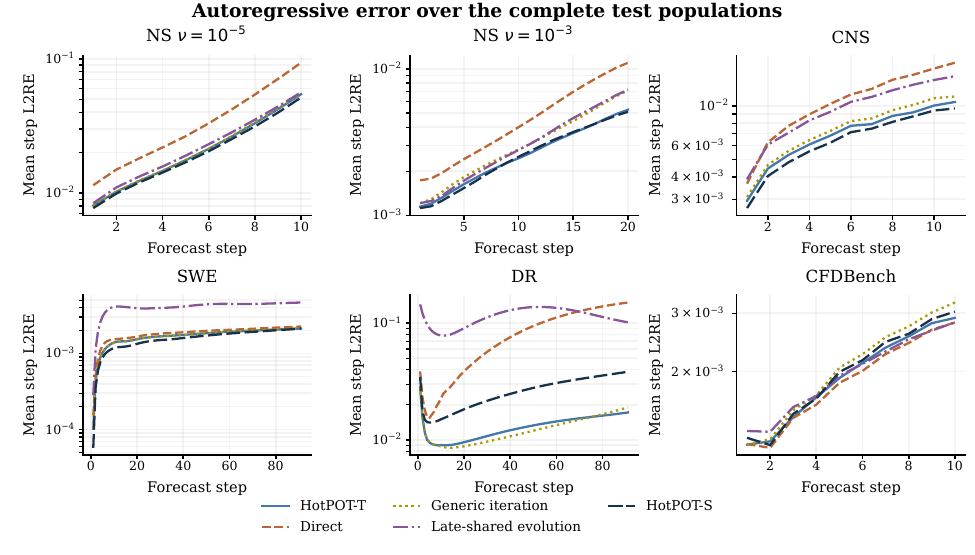}
\caption{\textbf{Mean per-step errors across the complete test populations.}
Each curve averages per-trajectory, scored-channel relative $L_2$ error
at each forecast step. All models use fixed E1000 checkpoints and their
own autoregressive feedback. CFDB uses the same row-weighted population
as the main table. Vertical axes are logarithmic.}
\label{fig:all_test_rollout}
\end{figure}

\Cref{fig:ns_generic_slices} shows where Full and Generic iteration differ
on four matched NS trajectories.

\begin{figure}[!htbp]
\centering
\includegraphics[width=\textwidth]{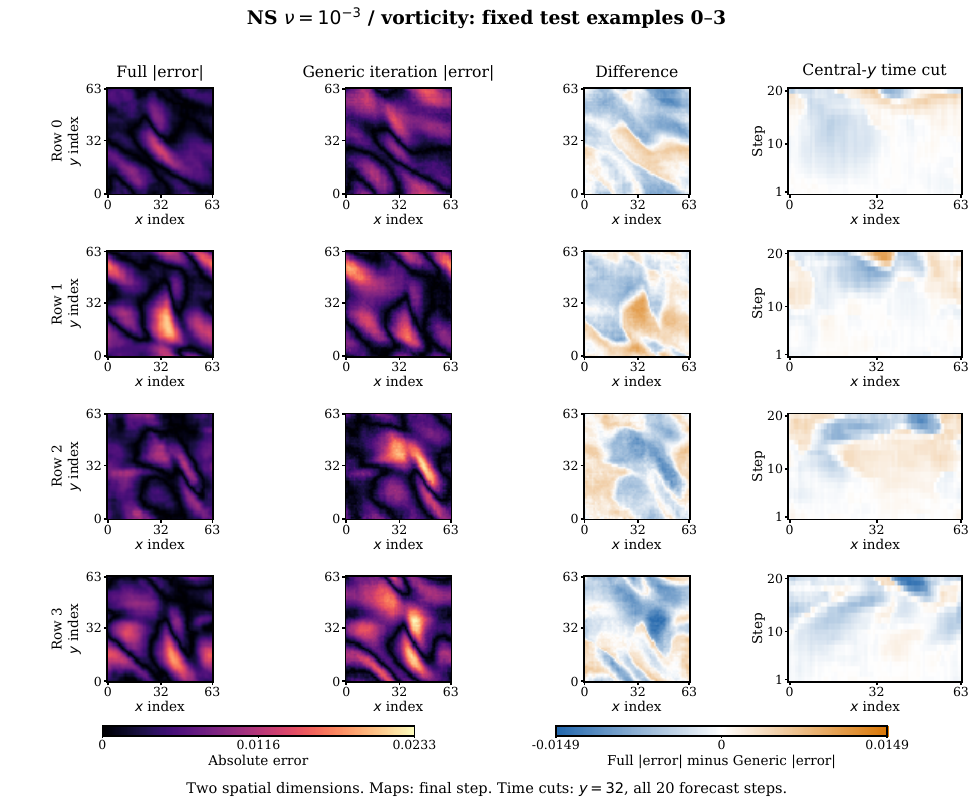}
\caption{\textbf{NS numerical execution and generic iteration.}
The first four test examples at $\nu=10^{-3}$ use matched reference
trajectories. The first two columns show final-step absolute errors;
the third shows their difference. The fourth is an $x$--time cut at the
central $y$ index through the difference volume. Blue favors Full, orange
favors Generic iteration. Error scales are shared across all four examples;
difference scales are symmetric and use the complete forecast volumes.
The task has two spatial dimensions; the rightmost vertical axis is time.}
\label{fig:ns_generic_slices}
\end{figure}

\subsection{Few-shot transfer to the wave family}
\label{app:wave}

\paragraph{Data and initialization.}
Wave-Layer and Wave-Gauss are two datasets from one wave-equation family
in the Poseidon benchmark \citep{herde2024poseidon}. Neither belongs to either source
checkpoint's six-dataset pre-training mixture. For each target, we use a
fixed custom split of mutually disjoint original trajectories: 1,800 in the
training pool, 100 for validation, and 200 for testing. Each original trajectory
belongs to exactly one split; sampled temporal windows remain within that split.
All methods use these same splits.
The two targets are evaluated
separately at $128\times128$ resolution with ten input frames and 11/6
forecast frames for Layer/Gauss. Nested subsets contain 16, 32, 64, or
128 training trajectories.

HotPOT transfer initializes the 2.012M-parameter field-wise history encoder
from the
17M E1000 source checkpoint. A scalar system-specific predictor is randomly
initialized, and all 16.771M target-model parameters are trained.
The scratch control has the same architecture and identical initial
predictor weights, differing only in encoder initialization. Both use a
common residual-prediction architecture that adds the predicted increment
to the current displacement.
Each model retains its native input normalization.
MoE-POT loads its full local 1000-epoch pre-training checkpoint, updates all its trainable parameters, including the
feature extractor and gating parameters.

\paragraph{Training and testing.}
All methods use the same target splits, sampled windows, training budget,
optimizer, learning-rate schedule, and validation-selection rule.
All runs use 500 epochs, batch $\min(20,\mathrm{shots})$, Adam
with $\beta_1=\beta_2=0.9$ and weight decay $10^{-6}$, and OneCycleLR
with maximum $10^{-4}$, 20\% warmup, and both division factors $10^4$.
Each epoch samples one ten-frame-to-next-frame window per trajectory;
the sampling indices and order are matched across arms. Validation
evaluates the full rollout every ten epochs. The lowest validation error
selects the checkpoint, with the earliest checkpoint breaking ties.
The selected model is evaluated on the 200 held-out test trajectories.

The metric is the trajectory-uniform mean of per-step spatial L2RE, with
denominator floor $10^{-8}$ and each model's own autoregressive feedback.

\begin{table}[!htb]
\centering\small
\setlength{\tabcolsep}{3.5pt}
\caption{Transfer to one unseen wave-equation family. Mean per-step L2RE on 200 custom held-out trajectories, averaged over 11/6 forecast frames for Layer/Gauss. All methods update their complete target models and select checkpoints by full-rollout validation within 500 epochs.}
\label{tab:wave_transfer}
\sisetup{detect-weight=true,mode=text,group-digits=false,table-number-alignment=center}
\begin{tabular}{ll S[table-format=1.3] S[table-format=1.3] S[table-format=1.3] S[table-format=1.3]}\toprule
Target & Model & {16} & {32} & {64} & {128} \\
\midrule
Wave-Layer & MoE-POT full fine-tuning & 0.576 & 0.517 & 0.435 & 0.365 \\
 & HotPOT scratch & 0.553 & 0.428 & 0.331 & 0.235 \\
 & HotPOT encoder transfer & \bfseries 0.506 & \bfseries 0.395 & \bfseries 0.304 & \bfseries 0.220 \\
\midrule
Wave-Gauss & MoE-POT full fine-tuning & 0.514 & 0.452 & 0.394 & 0.325 \\
 & HotPOT scratch & 0.542 & 0.372 & 0.284 & 0.225 \\
 & HotPOT encoder transfer & \bfseries 0.453 & \bfseries 0.370 & \bfseries 0.279 & \bfseries 0.218 \\
\bottomrule\end{tabular}
\end{table}

\paragraph{Results.}
Encoder transfer lowers error relative to MoE-POT full fine-tuning by
12.1--39.8\% on Wave-Layer and 11.9--32.9\% on Wave-Gauss.
It also improves on scratch training in all eight target--budget pairs,
with reductions of 6.5--8.4\% on Layer and 0.6--16.4\% on Gauss.
The comparison with scratch measures the effect of pre-trained encoder
initialization under this matched target architecture and training protocol.

\subsection{Physical diagnostics}
\label{app:physics}

For NS, relative enstrophy error compares the spatial mean squared
vorticity with the reference at each forecast step and divides by the
reference value, floored at $10^{-8}$. For SWE and CNS, conservation drift
measures the absolute change of a predicted spatial mean from the last
observed state. SWE normalizes by the initial absolute mean height; CNS
normalizes each conservative component by its initial spatial mean absolute
value. Metrics average all forecast steps and then all test trajectories.

\begin{table}[H]
\centering\small
\setlength{\tabcolsep}{5pt}
\caption{Physical diagnostics over the complete test rollouts of fixed E1000 Full and Direct models. Values average all trajectories and forecast steps; lower is better.}
\label{tab:physics}
\sisetup{detect-weight=true,mode=text,group-digits=false,table-number-alignment=center}
\begin{tabular}{ll S[table-format=1.2e-1] S[table-format=1.2e-1]}\toprule
System & Diagnostic & {Full} & {Direct} \\
\midrule
NS $10^{-5}$ & Relative enstrophy error & \bfseries 1.78e-3 & 3.63e-3 \\
NS $10^{-3}$ & Relative enstrophy error & \bfseries 2.16e-3 & 2.51e-3 \\
CNS & Normalized mass drift & \bfseries 7.63e-8 & 2.04e-4 \\
 & Normalized energy drift & \bfseries 8.15e-8 & 2.06e-4 \\
SWE & Relative mass drift & \bfseries 1.19e-9 & 9.31e-5 \\
\bottomrule
\end{tabular}
\end{table}

Full has lower mean enstrophy error on both NS settings and lower
conservation drift on SWE and CNS than Direct in \cref{tab:physics}.
Propositions~\ref{prop:swe} and~\ref{prop:cns} establish constraint preservation
in exact arithmetic; the table measures the implemented floating-point
rollouts. The CFDB program copies auxiliary support channels exactly.

\paragraph{Observed numerical execution.}
\label{app:execution_observations}
To distinguish mean correction from admissibility limiting, we observe the
actual E1000 Full program on all 60 SWE trajectories over 91 steps and all
200 CNS trajectories over 11 steps. At each step, we record the uncentered
neural proposal, the mean-corrected candidate, and the executed state along
the model's own feedback trajectory.

\begin{figure}[!htbp]
\centering
\includegraphics[width=\textwidth]{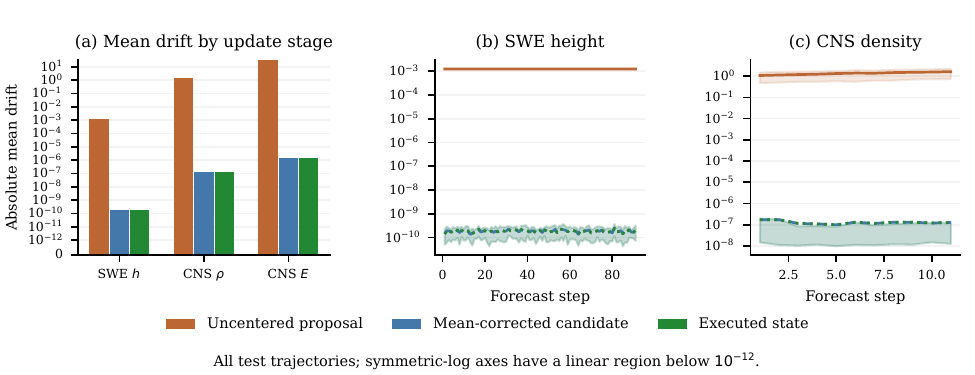}
\caption{\textbf{Numerical execution along the fixed E1000 HotPOT rollouts.}
(a) Mean absolute spatial-mean drift at each program stage, averaged over
all recorded trajectory--step pairs for SWE height, CNS density, and CNS energy. (b,c) Per-step drift for SWE height
and CNS density. Drift is the absolute change in spatial mean from the
current state to the candidate next state, in released field units.
Lines show arithmetic means and bands the 25th--75th percentiles across
trajectories. All stages use the same current state on the model's own
feedback trajectory.}
\label{fig:numerical_execution}
\end{figure}

\Cref{fig:numerical_execution} shows the drift before and after mean
correction. The limiter coefficient equals one at all 7,660 evaluated
updates, so each mean-corrected candidate is executed unchanged.
The minimum executed SWE height is 0.2519; the minimum CNS density
and pressure are 0.05715 and 0.3181, respectively.

\paragraph{SWE reference compatibility.}
\label{app:swe_reference}
Across all 60 reference trajectories and their 91 forecast frames,
float64 spatial reductions give the following compatibility measures.
Relative to the last observed
frame, the maximum absolute change in mean height is $1.20\times10^{-9}$.
After shifting each forecast reference to that initial mean, both the initial
and mean-shifted reference fields have minimum height above $0.2559$,
well above the $10^{-6}$ floor. The shifted references are compatible with
the mass-preserving update in Appendix~\ref{app:theory}. Let $\mathbf h^{\mathrm{ref},n}\in\R^N$ be the forecast reference
field, $\bar h^{\mathrm{ref},n}=\langle\mathbf h^{\mathrm{ref},n}\rangle$
its spatial mean, and $\bar h^{\mathrm{obs}}$ the mean of the last observed
input frame. The relative $L_2$ discrepancy associated with the initial mean,
$\sqrt N|\bar h^{\mathrm{ref},n}-\bar h^{\mathrm{obs}}|/
\|\mathbf h^{\mathrm{ref},n}\|_2$, averages
$2.27\times10^{-10}$ over trajectories and forecast steps.

\subsection{Paired field visualizations}
\label{app:qualitative}

\Cref{fig:qualitative_fields_a,fig:qualitative_fields_b} compare reference
fields, model predictions, and absolute errors across six settings,
including separate cavity and cylinder examples. Labels $u_x,u_y$ denote
velocity components in CNS and CFDB. All scored fields use
shared color scales within each row. Both 17M-active models use fixed E1000
checkpoints, the same ten-frame history, and autoregressive feedback.

\begin{figure}[H]
\centering
\includegraphics[width=\textwidth]{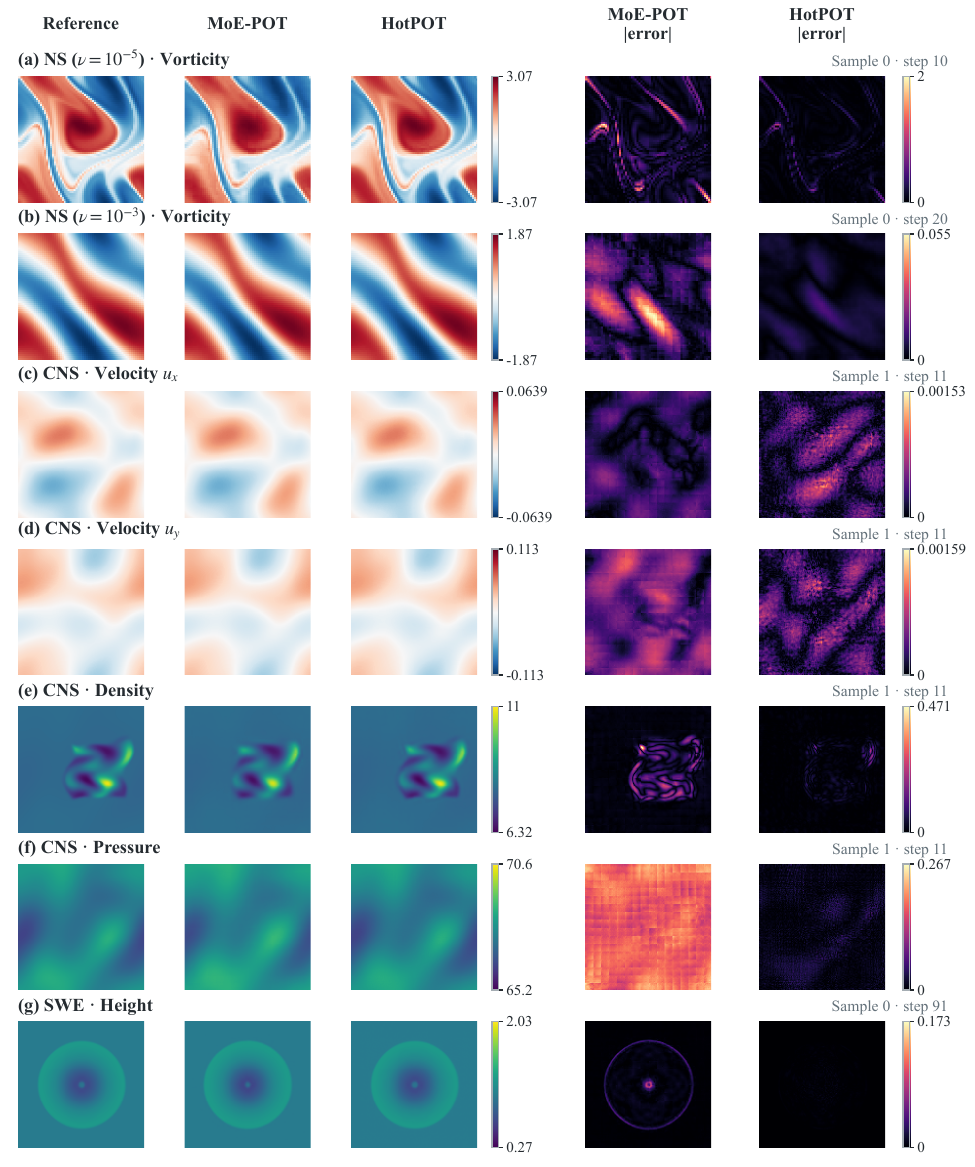}
\caption{\textbf{Predictions and absolute errors: NS, CNS,
and SWE.} All scored
fields are shown at the final forecast step (NS/SWE: row 0; CNS: row 1).
Reference and prediction fields share color limits within each row; error
maps share a zero lower bound and the full-rollout maximum. Signed fields
use zero-centered scales.}
\label{fig:qualitative_fields_a}
\end{figure}

\begin{figure}[H]
\centering
\includegraphics[width=\textwidth]{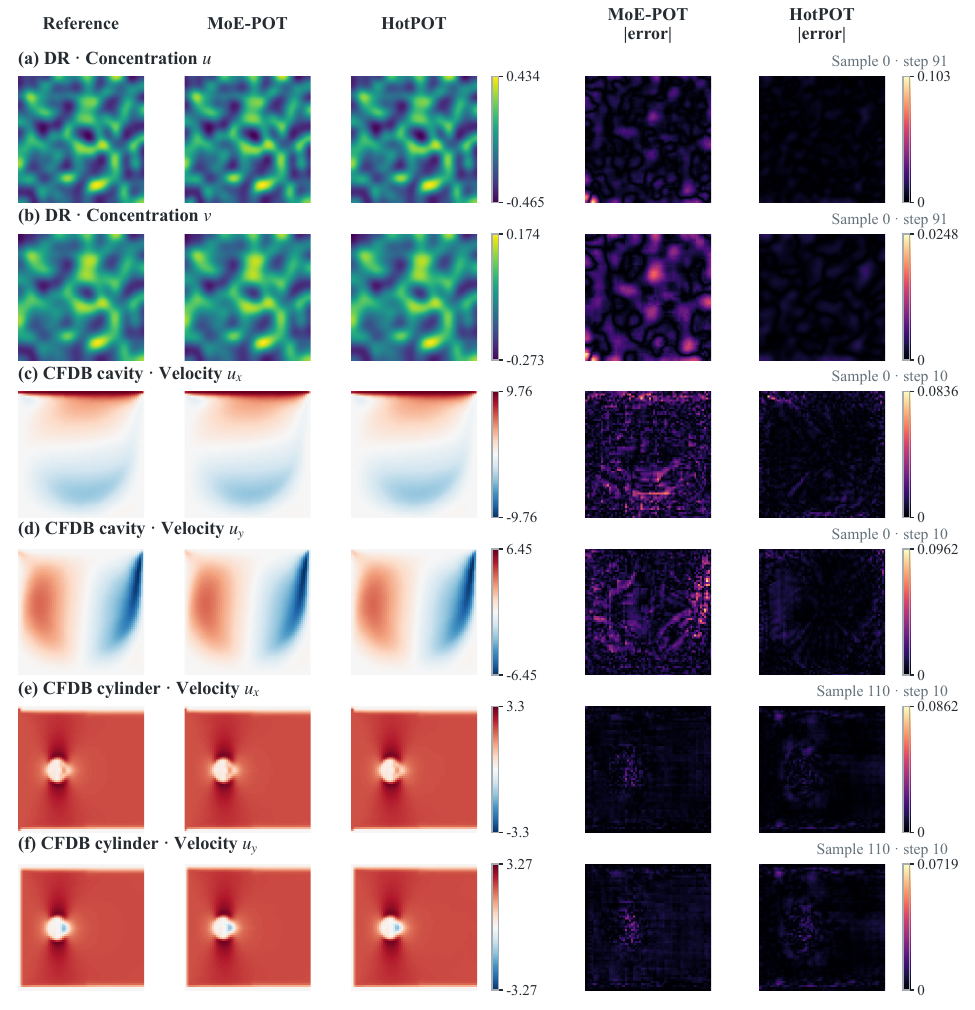}
\caption{\textbf{Predictions and absolute errors: DR and
CFDB.} All scored fields appear at the last
forecast step (DR/cavity: row 0; cylinder: row 110).
Within each row, reference and prediction fields share color limits; error
maps share a zero lower bound and the full-rollout maximum.}
\label{fig:qualitative_fields_b}
\end{figure}

\paragraph{Temporal error profiles.}
\Cref{fig:rollout_error_curves} follows the same seven examples through
every forecast step, complementing the final-frame maps with their temporal
error profiles.

\subsection{Space--time error structure}
\label{app:spacetime}

We examine the SWE and DR trajectories from Appendix~\ref{app:qualitative},
retaining all 91 forecast steps and both DR fields. Fixing zero-based grid
index $y=64$ gives $x$--time cuts with forecast time on the vertical axis
(\cref{fig:spacetime_slices}). SWE shows outward-moving and returning waves;
DR shows evolving concentration patterns. Color scales are shared within each field.

{\setlength{\intextsep}{6pt}\setlength{\abovecaptionskip}{4pt}
\begin{figure}[H]
\centering
\includegraphics[width=0.92\textwidth]{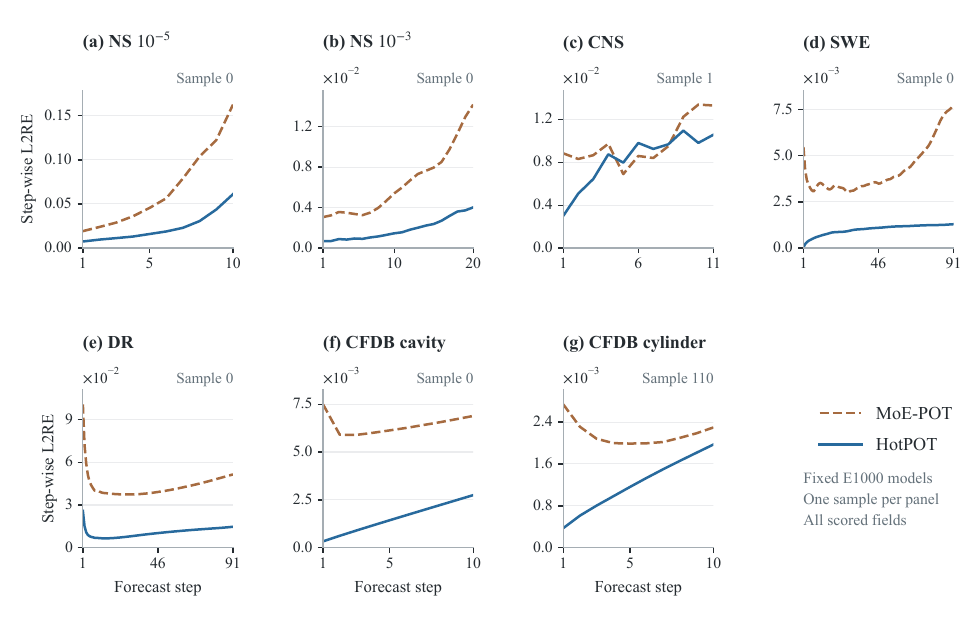}
\caption{\textbf{Temporal errors for the seven illustrated trajectories.}
Both E1000 models use autoregressive feedback from the same ten-frame histories.
Each unsmoothed curve averages spatial L2RE over scored channels.}
\label{fig:rollout_error_curves}
\vspace{6pt}
\includegraphics[width=0.92\textwidth]{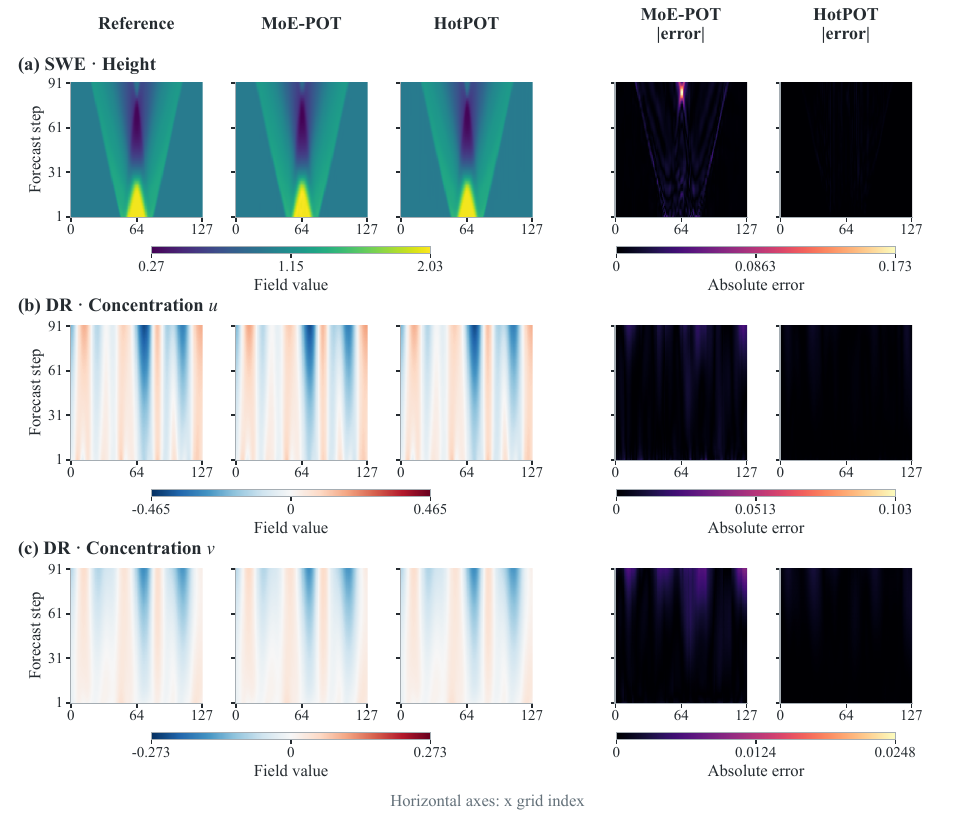}
\caption{\textbf{Space--time slices of SWE height and DR $u,v$ at zero-based $y=64$.}
Axes are $x$ grid index and forecast step. Reference/prediction colors share
full-array extrema within each row; error panels use zero to the full-array
maximum. Both examples use trajectory row 0.}
\label{fig:spacetime_slices}
\end{figure}}

\end{document}